%% file: main.tex
\documentclass{article}

\usepackage[nonatbib, final]{neurips_2024}

\usepackage[utf8]{inputenc} 

\RequirePackage[backend=biber, style=authoryear, citestyle=authoryear, maxcitenames=2, natbib=true]{biblatex}

\AtEveryBibitem{
	\ifentrytype{book}{}{
		\clearlist{publisher}
		\clearname{editor}
	}
}

\usepackage[T1]{fontenc}    
\usepackage{hyperref}       
\usepackage{url}            
\usepackage{booktabs}       
\usepackage{amsfonts}       
\usepackage{nicefrac}       
\usepackage{microtype}      
\usepackage{xcolor}         
\usepackage{algorithm}
\usepackage{algorithmic}

\usepackage{graphicx}

\usepackage{amsmath}
\usepackage{amssymb}
\usepackage{mathtools}
\usepackage{amsthm}
\usepackage{thmtools}

\usepackage{centernot}  
\usepackage{xfrac}      
\usepackage{bbm}        

\usepackage{enumitem}   
\usepackage[capitalize,noabbrev]{cleveref}

\theoremstyle{plain}
\newtheorem{theorem}{Theorem}[section]
\newtheorem{proposition}[theorem]{Proposition}
\newtheorem{lemma}[theorem]{Lemma}

\theoremstyle{definition}
\newtheorem{definition}[theorem]{Definition}

\theoremstyle{remark}

\usepackage[colorinlistoftodos,bordercolor=orange,backgroundcolor=orange!20,linecolor=orange,textsize=scriptsize]{todonotes}

\renewcommand{\bar}{\overline}

\usepackage{tikz}
\usepackage{pgfplots}
\pgfplotsset{}
\input{macros/plots}
\usetikzlibrary{shapes, arrows}
\usetikzlibrary{decorations.pathreplacing, calligraphy, calc}
\pgfplotsset{
	compat=1.16,
	oracle/.style={color=red, style=solid, line width=1.5pt},
	bound/.style={color=blue, style=solid, line width=1.5pt},
	objective/.style={color=black, style=solid, line width=1.5pt},
}

\definecolor{BurntOrange}{HTML}{CC5500}
\definecolor{DarkFern}{HTML}{407428}
\definecolor{CBRed}{HTML}{994F00}
\definecolor{CBBlue}{HTML}{006CD1}
\colorlet{Fern}{DarkFern!85!white}
\colorlet{LightFern}{DarkFern!20!white}

\colorlet{LightCerulean}{CBBlue!20!white}

\input{macros/math}


\title{Directional Smoothness and Gradient Methods: Convergence and Adaptivity}

\author{%
    Aaron Mishkin\thanks{Equal contribution.}\\
    Stanford University\\
    \texttt{amishkin@cs.stanford.edu}
   \And
   Ahmed Khaled\footnotesize{$^*$}\\
   Princeton University \\
   \texttt{ahmed.khaled@princeton.edu}
   \And
    Yuanhao Wang\\
   Princeton University \\
   \texttt{yuanhaoa@princeton.edu}
   \And
    Aaron Defazio\\
    FAIR, Meta AI\\
    \texttt{adefazio@meta.com}
   \And
   Robert M.~Gower \\
   CCM, Flatiron Institute \\
   \texttt{gowerrobert@gmail.com}
}

\begin{document}

\maketitle

\begin{abstract}
    We develop new sub-optimality bounds for gradient descent (GD) that depend on
    the conditioning of the objective along the path of optimization rather than
    on global, worst-case constants.
    Key to our proofs is directional smoothness, a measure of gradient variation
    that we use to develop upper-bounds on the objective.
    Minimizing these upper-bounds requires solving implicit equations to obtain a
    sequence of strongly adapted step-sizes; we show that these equations are
    straightforward to solve for convex quadratics and lead to new guarantees for
    two classical step-sizes.
    For general functions, we prove that the Polyak step-size and normalized GD
    obtain fast, path-dependent rates despite using no knowledge of the
    directional smoothness.
    Experiments on logistic regression show our convergence guarantees are tighter
    than the classical theory based on \( L \)-smoothness.
\end{abstract}

\section{Introduction}%
\label{sec:introduction}

\input{sections/intro}

\section{Directional Smoothness}%
\label{sec:local-direct-smoothn}

\input{sections/directional_smoothness}

\section{Path-Dependent Sub-Optimality Bounds}%
\label{sec:path-dependent-rates}

\input{sections/path_dependent_rates}

\section{Adaptive Learning Rates}%
\label{sec:adapt-learn-rates}

\input{sections/adaptive_learning_rates}

\section{Experiments}%
\label{sec:experiments}

\input{sections/experiments}

\section{Conclusion}
\label{sec:conclusion}

\input{sections/conclusion}

\newcommand{\acktext}{
    \section*{Acknowledgements}
    Aaron Mishkin was supported by NSF Grant DGE-1656518, by NSERC Grant
    PGSD3-547242-2020, and by an internship at the Center for Computational
    Mathematics, Flatiron Institute.
    We thank Si Yi Meng for insightful discussions during the preparation
    of this work and Fabian Schaipp for use of the
    \href{https://github.com/fabian-sp/step-back}{step-back} code.
    We also thank the anonymous reviewers for comments leading to improvements
    in \cref{prop:sc-iterates-analysis} and the addition of
    \cref{thm:polyak-alternate}.
}

\ifdefined\@neuripsfinaltrue
    \acktext
\fi

\printbibliography[]

\newpage
\appendix

\section{Proofs for Section~\ref{sec:local-direct-smoothn}}%
\label{app:directional-smoothness-proofs}

\input{appendices/directional_smoothness_proofs}

\section{Proofs for Section~\ref{sec:path-dependent-rates}}%
\label{app:path-dependent-rates-proofs}

\input{appendices/path_dependent_rates_proofs}

\section{Proofs for Section~\ref{sec:quadratic-case}}%
\label{app:quadratic-proofs}

\input{appendices/quadratic_case_proofs}

\section{Proofs for Section~\ref{sec:adaptivity-general-case}}%
\label{app:general-case-proofs}

\input{appendices/general_case_proofs}

\section{Experimental Details}%
\label{app:experimental-details}

\input{appendices/experiments}

\section{Computational Details}%
\label{app:computational-details}

\input{appendices/computation}




\end{document}

%% file: macros/plots.tex
\usepackage{tikz}
\usepackage{pgfplots}

\usepgfplotslibrary{fillbetween}
\usetikzlibrary{patterns}

\tikzset{
    font={\fontsize{12pt}{12}\selectfont},
}

\pgfplotsset{
    compat=1.5.1,
    primary/.style={color=black, style=solid, line width=1.5pt}, 
    secondary/.style={color=red, style=solid, line width=1.5pt}, 
}

%% file: macros/math.tex
\def\bfI{\mathbf{I}}

\def\bbN{\mathbb{N}}

\def\calB{\mathcal{B}}
\def\calC{\mathcal{C}}

\def\calG{\mathcal{G}}

\def\calI{\mathcal{I}}
\def\calJ{\mathcal{J}}

\usepackage{thmtools, thm-restate}

\usepackage{mathtools} 

\DeclarePairedDelimiter{\ceil}{\lceil}{\rceil}

\newcommand{\abs}[1]{\left\vert #1\right\vert}

\newcommand{\rbr}[1]{\left(#1\right)}
\newcommand{\sbr}[1]{\left[#1\right]}
\newcommand{\cbr}[1]{\left\{#1\right\}}
\newcommand{\abr}[1]{\left\langle#1\right\rangle}

\def\norm#1{\|#1\|}

\def\argmin{\mathop{\rm arg\,min}}
\newcommand{\sign}{\text{sign}}

\def\half{\frac 1 2}

\newcommand{\R}{\mathbb{R}}

\newcommand{\yk}{y_k}
\newcommand{\ykk}{y_{k+1}}

\newcommand{\vk}{v_k}
\newcommand{\vkk}{v_{k+1}}

\newcommand{\wopt}{w^*}

\newcommand{\x}{x}
\newcommand{\xk}{x_k}

\newcommand{\xkk}{x_{k+1}}
\newcommand{\xopt}{x^*}
\newcommand{\xbar}{\xbar{w}}

\newcommand{\zk}{z_{k}}

\newcommand{\etamin}{\eta_{\text{min}}}

\newcommand{\etak}{\eta_k}
\newcommand{\etakk}{\eta_{k+1}}

\newcommand{\grad}{\nabla f}

\DeclarePairedDelimiterX{\infdivx}[2]{\lparen}{\rparen}{%
  #1\;\delimsize\|\;#2%
}

\newcommand\ev[1]{\left \langle #1 \right \rangle}

\newcommand{\pr}[1][]{
  \ifthenelse { \equal{#1}{} }
  { \ensuremath{\mathbb{P}} }
  { \ensuremath{\mathbb{P}\left(#1\right)} }
}

\renewcommand{\R}{\mathbb{R}}

\newcommand{\sqn}[1]{{\left\lVert#1\right\rVert}^2}

\newcommand{\eqdef}{\overset{\text{def}}{=}}


%% file: sections/intro.tex
Gradient methods for differentiable functions are typically analyzed under the
assumption that \( f \) is \( L \)-smooth, meaning \( \nabla f \) is
\( L \)-Lipschitz continuous.
This condition implies \( f \) is upper-bounded by a quadratic and guarantees
that gradient descent (GD) with step-size \( \eta < 2 / L \) decreases the
optimality gap at each iteration \citep{bertsekas1997nonlinear}.
However, experience shows that GD can still decrease the objective when
\( f \) is not \( L \)-smooth, particularly for deep neural networks
\citep{bengio2012pratical, li2020reconciling, cohen2021stability}.
Even for functions verifying smoothness, convergence rates are often
pessimistic and fail to predict optimization speed in practice
\citep{paquette2023halting}.

One alternative to global smoothness is local Lipschitz continuity of the
gradient (``local smoothness'').
Local smoothness assumes different Lipschitz constants hold for different
neighbourhoods, which avoids global assumptions and improves rates.
However, such analyses typically rely on boundedness of the iterates and then
use local smoothness to obtain \( L \)-smoothness over a compact set
\citep{malitsky2020descent}.
Boundedness is guaranteed in several ways: \citet{zhang2020first} break
optimization into stages, \citet{patel2022gradient} use stopping-times,
and \citet{lu2023accelerated} employ a line-search.
Unfortunately, these approaches modify the underlying optimization algorithm,
require local smoothness oracles \citep{park2021preconditioned}, or rely on
highly complex arguments.

In contrast, we prove simple rates for GD without global smoothness by
deriving bounds of the form,
\begin{equation}
    \label{eq:quadratic-bound}
    \begin{aligned}
        f(\xkk) & \leq f(\xk) + \abr{\grad(\xk), \xkk - \xk}
        + \frac{M(\xkk, \xk)}{2} \norm{\xkk - \xk}_2^2,
    \end{aligned}
\end{equation}
where the \emph{directional smoothness function} \( M(\xkk, \xk) \) depends only on
properties of \( f \) along the chord between \( \xk \) and \( \xkk \).
Our sub-optimality bounds provide a path-dependent perspective on GD and are tighter than
conventional analyses when the step-size sequence is adapted to the
directional smoothness, meaning \( \etak < 2 / M(\xkk, \xk) \).
See \cref{fig:rate-comparison} for two real-data examples highlighting our
improvement over classical rates.
We summarize all our contributions as follows.

\begin{figure*}[t]
    \centering
    \includegraphics[width=0.98\textwidth]{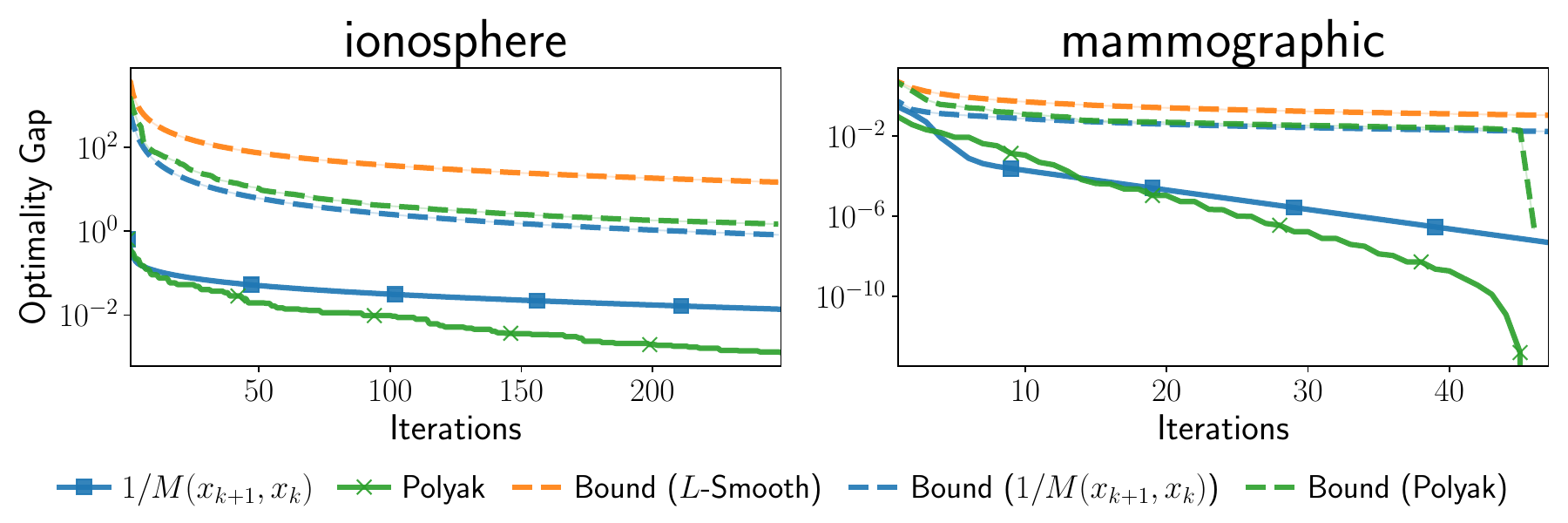}
    \vspace{-2ex}
    \caption{
        Comparison of actual (solid lines) and theoretical (dashed lines)
        convergence rates for GD with (i) step-sizes strongly adapted to the
        directional smoothness (\( \etak = 1 / M(\xkk, \xk) \)) and (ii) the
        Polyak step-size.
        Both problems are logistic regressions on UCI repository datasets
        \citep{asuncion2007uci}.
        Our bounds using directional smoothness are tighter than those based on
        global \( L \)-smoothness of \( f \) and adapt to the optimization path.
        For example, on \texttt{mammographic} our theoretical rate for the Polyak
        step-size concentrates rapidly exactly when the optimizer shows fast
        convergence.
    }%
    \label{fig:rate-comparison}
    \vskip -0.2in
\end{figure*}


\textbf{Directional Smoothness.}
We introduce three constructive directional smoothness functions \( M(x, y) \).
The first, point-wise smoothness, depends only on the end-points \( x, y \) and
is easily computed, while the second, path-wise smoothness, yields a tighter
bound, but depends on the chord \( \calC = \cbr{\alpha x + (1- \alpha)y :
    \alpha \in [0, 1]} \).
The last function, which we call the optimal point-wise smoothness, is both
easy-to-evaluate and provides the tightest possible quadratic upper bound.

\textbf{Sub-optimality bounds.}
We leverage directional smoothness functions to prove new sub-optimality bounds
for GD on convex functions.
Our bounds are localized to the GD trajectory, hold for any step-size sequence,
and are tighter than the classic analysis using \( L \)-smoothness.
They are also more general since we do not need to assume that $f$ is globally
$L$-smooth to show progress;
all we require is a sequence of step-sizes adapted to the directional
smoothness function.
Furthermore, our approach extends naturally to acceleration, allowing us
to prove optimal rates for (strongly)-convex functions.

\textbf{Adaptive Step-Sizes in the Quadratic Case.}
In the general setting, computing step-sizes which are adapted to the
directional smoothness requires solving a challenging non-linear root-finding
problem.
For quadratic problems, we show that the ideal step-size that satisfies \(
\etak = 1 / M(\xkk, \xk) \) is the Rayleigh quotient and is connected to the
hedging algorithm \citep{altschuler2023hedging}.
%

\textbf{Exponential Search.}
Moving beyond quadratics, we prove that the equation \( \etak = 1 / M(\xkk,
\xk) \) admits a solution under mild conditions, meaning ideal
step-sizes can be computed using Newton's method.
Since computing these step-sizes is typically impractical, we adapt
exponential search \citep{yair2022parameter} to obtain similar path-dependent
complexities up to a log-log penalty.

\textbf{Polyak and Normalized GD.}
More importantly, we show that the Polyak step-size
\citep{polyak1987introduction} and normalized GD achieve fast, path-dependent
rates \emph{without} knowledge of the directional smoothness.
Our analysis reveals that the Polyak step-size adapts to \emph{any}
directional smoothness to obtain the tightest possible convergence rate.
This property is not shared by constant step-size GD and may explain the
superiority of the Polyak step-size in many practical settings.



\subsection{Additional Related Work}\label{sec:related-work}

Directional smoothness is a relaxation of non-uniform smoothness
\citep{mei21_lever_non_unifor_first_order}, which restricts the smoothness
function \( M \) to depend only on \( x \), the origin point.
\citet{mei21_lever_non_unifor_first_order} leverage
non-uniform smoothness and a non-uniform Łojasiewicz inequality
to break lower-bounds for first-order optimization.
Similarly, \citet{berahas23_non_unifor_smoot_gradien_descen} show that a weak
local smoothness oracle can break lower bounds for gradient methods.
A major advantage of our work over such oracle-based approaches is that we
construct explicit directional smoothness functions that are easy to evaluate.

Similar to non-uniform smoothness, \citet{grimmer2019convergence} and
\citet{orabona23_normal_gradien_all} consider H{\"o}lder-type growth conditions
with constants that depend on a neighbourhood of \( x \).
Since directional smoothness is stronger than and implies these H{\"o}lder
error bounds, our \( M \) functions can be leveraged to make their results
fully explicit (the H{\"o}lder bounds are non-constructive).
Finally, while they also analyze normalized GD, our rates are anytime and do
not use online-to-batch reductions like \citet{orabona23_normal_gradien_all}.

Directional smoothness is also related to \( (L_0, L_1) \)-smoothness
\citep{zhang2020clipping, zhang2020improved}, which can be interpreted as a
directional smoothness function with exponential dependence on the distance
between \( x \) and \( y \).
The extension of \( (L_0, L_1) \)-smoothness to \( (r, l) \)-smoothness by
\citet{li2023generalized} shows how to bound sequences of such directional
smoothness functions, even for accelerated methods.
These approaches are complementary to ours and showcase a setting where
directional smoothness leads to concrete convergence rates.

Our work is most closely connected to that by \citet{malitsky2020descent}, who use
a smoothed version of \( M(x, y) \) to set the step-size.
\citet{vladarean2021adaptivity} apply a similar smoothed step-size scheme to
primal-dual hybrid gradient methods, while \citet{zhao2024adaptive}
relate directional smoothness to Barzilai-Borwein updates
\citep{barzilai1988two} and
\citet{vainsencher2015local} use local smoothness over neighbourhoods of the
global minimizer to set the step-size for SVRG.\

Finally, we note that adaptivity to directional smoothness is different from
adaptivity to the sequence of observed gradients obtained by methods such as
Adagrad ~\citep{duchi11_adagrad,streeter10_less_regret_via_onlin_condit}.
Adagrad and its variants are most useful when the gradients are bounded, such
as in Lipschitz optimization, although they can also be used to obtain rates
for smooth functions~\citep{levy17_onlin_to_offlin_conver_univer}.
We do not address adaptivity to gradients in this work.


%% file: sections/directional_smoothness.tex

We say that a convex function $f$ is \( L \)-smooth if for all $x, y \in
    \mathbb{R}^d$,
\begin{equation}\label{eq:smoothness-upper-bound}
    f(y) \leq f(x) + \abr{\grad(x), y - x} +\frac{L}{2}\norm{y - x}_2^2.
\end{equation}
Minimizing this quadratic upper bound in \( y \)
gives the classical GD update with step-size \( \etak = 1 / L \).
However, this viewpoint leads to rates which depend on the global, worst-case
growth of \( f \).
This is both counter-intuitive and undesirable because the iterates of GD,
\( \xkk = \xk - \etak \grad(\xk), \)
depend only on local properties of \( f \).
Ideally, the analysis should also depend only on the local conditioning
along the path \( \cbr{x_1, x_2, \ldots} \).
Towards this end, we generalize the smoothness upper-bound as follows.

\begin{definition}\label{def:directional-smoothness-function}
    We call
    $M: \mathbb{R}^{d,d} \rightarrow \mathbb{R}_+$
    a \emph{directional smoothness function} for
    $f$ if for all $x, y \in \mathbb{R}^d$,
    \begin{align}
        \label{eq:directional-smoothness-function-definition}
        f(y) \leq f(x) \!+\! \ev{ \nabla f(x) , y \!-\! x }
        +\frac{M(x, y)}{2} \sqn{y \!-\! x}.
    \end{align}
\end{definition}

If a function is $L$-smooth, then $M(x, y) = L$ is a trivial choice of
directional smoothness function.
In the rest of this section, we construct different \( M \) functions that
provide tighter bounds on \( f \) while still being possible to evaluate.
The first is the \emph{point-wise directional smoothness},
\begin{equation}\label{eq:directional-smoothness}
    D(x, y) := \frac{2 \norm{\grad(y) - \grad(x)}_2}{\norm{y - x}_2}.
\end{equation}
Point-wise smoothness is a directional estimate of \( L \) and
satisfies \( D(x, y) \leq 2L \).
Indeed, $L$ can be equivalently  defined as the supremum of $D(x, y) / 2$
over the domain  of \( f \) \citep{beck17_first_order_methods_opt}.
If $f$ is convex and differentiable, then $D(x, y)$ is a directional smoothness function according to
\Cref{def:directional-smoothness-function}.
\begin{restatable}{lemma}{directionalSmoothness}\label{lemma:directional-smoothness}
    If \( f \) is convex and differentiable, then the point-wise directional
    smoothness satisfies,
    \begin{equation}\label{eq:directional-smoothness-bound}
        f(y) \leq f(x) + \abr{\grad(x), y - x} + \frac{D(x, y)}{2} \norm{y - x}_2^2.
    \end{equation}
\end{restatable}
See \cref{app:directional-smoothness-proofs}
(we defer all proofs to the relevant appendices).
In the worst-case, the point-wise directional smoothness $D$ is weaker
than the standard upper-bound $M(x, y) = L$ by a factor of two.
This is \emph{not} an artifact of the analysis and is generally
unavoidable, as the next proposition shows.
\begin{restatable}[]{proposition}{dirSmoothTight}\label{prop:dirSmoothTight}
    There exists a convex, differentiable \( f \)
    and \( x, y \in \R^d \) such that if \( t < 2 \), then
    \begin{align}
        \label{eq:dir-smoothness-tight-prop}
        \begin{split}
            f(y) > f(x) & + \ev{\nabla f(x), y-x} +
            \frac{t\norm{\nabla f(x) - \nabla f(y)}}{2\norm{y-x}_2} \norm{y-x}^2_2.
        \end{split}
    \end{align}
\end{restatable}
While the point-wise smoothness is easy to compute, this additional factor of
two can make \cref{eq:directional-smoothness-bound} looser than
\( L \)-smoothness --- on isotropic quadratics, for example.
As an alternative, we define the \emph{path-wise directional smoothness},
\begin{equation}\label{eq:path-smoothness}
    A(x,y)
    := \sup_{t\in [0,1]} \frac{\abr{\grad(x \!+\! t(y \!-\! x)) \!-\! \grad(x), y \!-\! x} }
    {t\norm{y \!-\! x}^2}, \hspace{-0.2cm}
\end{equation}
and show it verifies the quadratic upper-bound and satisfies
\Cref{def:directional-smoothness-function} even without convexity.
\begin{restatable}{lemma}{pathSmoothness}\label{path-smoothness}
    For any differentiable function $f$, the path-wise smoothness~\eqref{eq:path-smoothness} satisfies
    \begin{equation}\label{eq:path-smoothness-bound}
        f(y) \leq f(x) + \abr{\grad(x), y - x} + \frac{A(x, y)}{2} \norm{y - x}_2^2.
    \end{equation}
\end{restatable}
Path smoothness is tighter than point-wise smoothness since
\( A(x, y) \leq D(x, y) \), but hard to compute because it depends on
the chord between \( x \) and \( y \).
That is, it depends on the properties of $f$ on the line
\( \left\{ tx + (1-t)y : t \in [0, 1] \right\} \)
rather than solely on the points $x$ and $y$ like the point-wise smoothness.

Point-wise and path-wise smoothness are constructive, but
they may not yield the tightest bounds in all situations.
The tightest directional smoothness function, which we call the \emph{optimal
    point-wise smoothness}, is the smallest number for which the
quadratic upper bound holds,
\begin{align}
    H(x, y) & = \frac{\abs{f(y)-f(x)-\ev{ \nabla f(x) , y-x }}}{\frac{1}{2} \sqn{y-x}} 
    \label{eq:optimal-directional-smoothness}
\end{align}
By definition, $H$ is the tightest possible directional
smoothness function; it lower bounds any constant $C$ that
satisfies the quadratic bound~\eqref{eq:smoothness-upper-bound}.
Thus, $H(x, y) \leq M(x, y)$ for any smoothness function $M$.

The directional smoothness functions introduced in this section represent
different trade-offs between computability and tightness.
The optimal point-wise smoothness $H(x, y)$ requires access to both the
function and gradient values, whereas the point-wise directional-smoothness
$D(x, y)$ requires only access to the gradients and convexity.
In contrast, the path-wise direction smoothness $A(x, y)$ satisfies
\Cref{path-smoothness} with or without convexity, but may be hard to evaluate.

%% file: sections/path_dependent_rates.tex

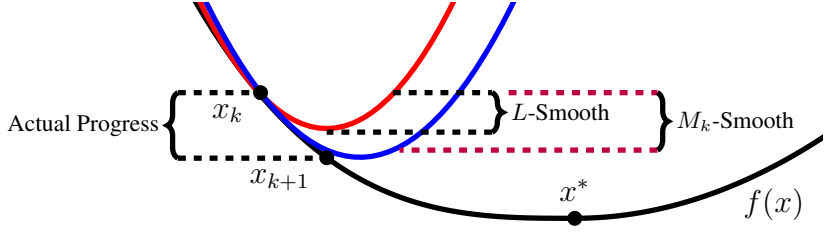
\begin{figure*}[t]
    \centering
    \input{figures/quadratic_bound}
    \caption{Illustration of GD with \( \etak = 1 / L \).
        Even though this step-size exactly minimizes the upper-bound from \( L
        \)-smoothness, \( M_k \) directional smoothness better predicts the progress
        of the gradient step because \( M_k \ll L \).
        Our rates improve on \( L \)-smoothness because of this tighter bound.
    }
    \label{fig:quadratic-bound}
\end{figure*}

Using directional smoothness, we obtain a descent lemma which depends only on
local geometry,
\begin{equation}
    \label{eq:local-descent-lemma}
    f(\xkk) \leq f(\xk) -
    \rbr{\etak - \frac{\etak^2 M( \xk, \xkk)}{2}} \norm{\grad(\xk)}_2^2.
\end{equation}
See \Cref{lemma:local-descent-lemma}.
If \( \etak < 2 / M( \xk, \xkk) \), then GD is guaranteed to decrease the
function value and we call \( \etak \) \emph{adapted} to $M( \xk, \xkk)$.
However, computing adapted step-sizes is not always straightforward.
For instance, finding \( \etak = 1 / M(\xk, \xkk(\etak)) \) requires solving a
non-linear equation.


The rest of this section leverages directional smoothness to derive new
guarantees for GD with arbitrary step-sizes.
We emphasize that these results are \emph{sub-optimality bounds},
rather than convergence rates;
a sequence of adapted step-sizes is required to convert our propositions into a
convergence theory.
As a trade-off, our bounds reflect the locality of GD, rather than treating it
as a global method.

We start with the case when \( f \) has lower curvature.
Instead of using strong convexity or the PL-condition \citep{karimi2016linear},
we propose the directional strong convexity constant,
\begin{equation}
    \label{eq:directional-sc-constant}
    \mu(x, y)
    \!=\! \inf_{t\in [0,1]} \frac{\abr{\grad(x \!+\! t(y \!-\! x)) \!-\! \grad(x), y \!-\! x} }
    {t\norm{y - x}_2^2}.
\end{equation}
If \( f \) is convex, then \( \mu(x, y) \geq 0 \) and it verifies the
standard lower-bound from strong convexity,
\begin{equation}
    \label{eq:directional-sc}
    f(y) \geq f(x) + \abr{\grad(x), y - x} + \frac{\mu(x, y)}{2}\norm{y - x}_2^2.
\end{equation}
Moreover, we have \( \mu(x, y) \geq \mu \) when \( f \) is $\mu$--strongly
convex.
We prove two bounds for convex functions using directional strong convexity.
For brevity, we denote \( M_i := M(\x_{i}, \x_{i+1}) \), \( \mu_i :=
\mu_i(x_i, \xopt) \), \( \delta_i = f(x_i) - f(\xopt) \), and \( \Delta_i =
\norm{x_i - \xopt}_2^2 \), where \( \xopt \) is a minimizer of \( f \).

\begin{restatable}{proposition}{scSplitAnalysis}\label{prop:sc-split-analysis}
    If \( f \) is convex and differentiable, then GD with step-size
    sequence \( \cbr{\etak} \) satisfies,
    \begin{equation}
        \label{eq:sc-split-analysis}
        \begin{split}
            \delta_k
             & \leq \sbr{\prod_{i \in \calG} \rbr{1 + \eta_i \lambda_i \mu_i}}
            \delta_0
            + \sum_{i \in \calB} \sbr{\prod_{j > i, j \in \calG} \rbr{1 + \eta_j \lambda_j \mu_j}}
            \frac{\eta_i \lambda_i}{2} \norm{\grad(\x_i)}_2^2,
        \end{split}
    \end{equation}
    where \( \lambda_i \!=\! \eta_i M_i \!-\! 2 \),
    \( \calG = \cbr{i : \eta_i \!<\! 2/M_i} \),
    and \( \calB = [k] \!\setminus\! \calG \).
\end{restatable}
The analysis splits iterations into good steps $\mathcal{G}$, where \( \etak
\) is adapted to the directional smoothness, and bad steps $\calB$, where the
step-size is too large and GD may increase the optimality gap.
When \( f \) is \( L \)-smooth and \( \mu \)-strongly convex,
using the step-size sequence \( \etak = 1 / L \) gives
\begin{equation}
    \label{eq:sc-tighter-rate}
    f(\xkk) - f(\xopt)
    \leq  \left[ \prod_{i = 0}^{k}
        \rbr{1 - \frac{\mu_i \rbr{2 - M_i / L}}{L}} \right] \rbr{f(\x_0) - f(\xopt)}
\end{equation}
where \( \mu_i \rbr{2 - M_i / L} \geq \mu \).
Thus, \Cref{eq:sc-split-analysis} gives at least as tight a rate
as standard assumptions by localizing to the convergence path using
\emph{any} directional smoothness $M$.
When \( M_i < L \), the gap in constants yields a strictly improved rate
(see \cref{fig:quadratic-bound}).
We also prove a more elegant bound.

\begin{restatable}{proposition}{scIteratesAnalysis}\label{prop:sc-iterates-analysis}
    If $f$ is convex and differentiable, then GD with step-size
    sequence \(\cbr{\etak} \) satisfies,
    \begin{equation}
        \label{eq:sc-iterates-analysis-eq}
        \begin{split}
             & \Delta_{k}
            \leq \sbr{\prod_{i=0}^{k} \frac{\abs{1 - \mu_i \eta_i}}{1 + \mu_{i+1} \eta_i}} \Delta_0
            + \sum_{i=0}^k \sbr{\prod_{j > i} \frac{\abs{1 - \mu_j \eta_j}}{1 + \mu_{j+1}\eta_j}}
            \frac{\rbr{M_i \eta_i^3 - \eta_i^2}}{1 + \mu_{i+1}\eta_i} \norm{\grad(\x_i)}_2^2.
        \end{split}
    \end{equation}
\end{restatable}
Unlike \Cref{prop:sc-split-analysis}, this analysis shows linear progress at
each iteration and does not divide \( k \) into good steps and bad steps.
In exchange, the second term in \Cref{eq:sc-iterates-analysis-eq} reflects how
much convergence is degraded when \( \etak \) is not adapted to the
directional smoothness function $M$.
We conclude this section with a bound for when there is no lower curvature,
meaning \( \mu_i = 0 \).
\begin{restatable}{proposition}{convexAnyStepsize}\label{prop:convex-any-stepsize}
    Let \( \bar x_k \!=\! \sum_{i=0}^k \eta_i x_{i+1} /\! \sum_{i=0}^k \eta_i \).
    If \( f \) is convex and differentiable, then GD satisfies,
    \begin{equation}
        \label{eq:1}
        \begin{split}
            f(\bar x_k) - f(\xopt)
             & \leq \frac{\norm{x_0 - \xopt}_2^2}{2\sum_{i=0}^k \eta_i}
            + \frac{\sum_{i=0}^{k} \eta_i^2 (\eta_i M_i - 1)\norm{\grad(\x_i)}_2^2}{2\sum_{i=0}^k \eta_i}.
        \end{split}
    \end{equation}
\end{restatable}

Eq.~\eqref{eq:1} is faster than standard analyses whenever $M_i < L$;
it will be a key tool in the next sections.

\subsection{Path-Dependent Acceleration}\label{sec:acceleration}

Now we show that directional smoothness can also be used to derive
path-dependent sub-optimality bounds for accelerated algorithms --- that is, methods
obtaining optimal rates for smooth, convex optimization.
In particular, we study Nesterov's accelerated gradient descent (AGD)
\citep{nesterov1983method} and prove that directional smoothness leads to
tighter rates given adapted step-sizes.
Throughout this section we assume that \( f \) is \( \mu \)-strongly convex
with \( \mu = 0 \) when \( f \) is merely convex.

Although our analysis uses estimating sequences \citep{nesterov2018lectures},
we state AGD in the following ``momentum'' formulation, where \( \yk \) is the
momentum and \( \alpha_k \) the momentum parameter,
\begin{equation}\label{eq:agd}
    \begin{aligned}
        \xkk
         & = \yk - \etak \grad(\yk)                                                               \\
        \alpha_{k+1}^2
         & = (1 - \alpha_{k+1}) \alpha_k^2 \frac{\eta_{k+1}}{\etak} + \eta_{k+1} \alpha_{k+1} \mu \\
        \ykk
         & = \xkk + \frac{\alpha_k (1 - \alpha_k)}{\alpha_k^2 + \alpha_{k+1}}\rbr{\xkk - \xk}.
    \end{aligned}
\end{equation}
If \( \etak \leq 1/M( \xk, \xkk) \), then \cref{eq:local-descent-lemma} combined
with \( 1 - \etak M( \xk, \xkk)/2 \geq 1/2 \) implies,
\begin{equation}\label{eq:agd-descent-condition}
    f(\xkk) \leq f(\yk) - \frac{\etak}{2} \norm{\grad(\yk)}_2^2.
\end{equation}
Our analysis leverages the fact that this descent condition for \( \xkk \) is
the only connection between the smoothness of \( f \) and the convergence rate
of AGD.\
Since \cref{eq:agd-descent-condition} depends only on the step-size \( \etak \),
we can replace \( L \) within the analysis of AGD with a sequence of adapted step-sizes.
The following theorem controls the effect of these step-sizes to obtain
path-dependent bounds.
\begin{restatable}{theorem}{adaptedAcceleration}\label{thm:adapted-acceleration}
    Suppose \( f \) is differentiable, $\mu$--strongly convex
    and AGD is run with  adapted
    step-sizes \( \etak \leq 1 / M_k \).
    If \( \mu > 0 \) and \( \alpha_0 = \sqrt{\eta_0 \mu} \), then AGD obtains
    the following accelerated rate:
    \begin{equation}\label{eq:sc-accelerated-rate}
        f(\xkk) - f(\xopt) \leq \prod_{i=0}^k \rbr{1 - \sqrt{\mu \eta_i}}
        \sbr{f(\x_0) - f(\xopt) + \frac{\mu}{2}\norm{\x_0 - \xopt}_2^2}.
    \end{equation}
    Let \( \etamin = \min_{i \in [k]} \eta_i \).
    If \( \mu \geq 0 \) and \( \alpha_0 \in (\sqrt{\mu \eta_0}, c) \),
    where \( c \) is the maximum value of \( \alpha_0 \) for which
    \( \gamma_0 = \frac{\alpha_0^2 - \eta_0 \alpha_0 \mu}{\eta_0 (1 - \alpha_0)} \)
    satisfies \( \gamma_0 < 3 / \etamin + \mu \),
    then AGD obtains the following rate:
    \begin{equation}\label{eq:convex-accelerated-rate}
        f(\xkk) - f(\xopt) \leq \frac{4}{\etamin (\gamma_0 - \mu) (k+1)^2}
        \sbr{f(\x_0) - f(\xopt) + \frac{\gamma_0}{2}\norm{\x_0 - \xopt}_2^2}.
    \end{equation}
\end{restatable}

If \( \etak = 1 / M_k > 1 / L \), then these rates are strictly faster than
those obtained under \( L \)-smoothness and \cref{thm:adapted-acceleration}
shows that AGD provably benefits from taking the largest possible steps given
the local geometry of \( f \).
However, obtaining accelerated rates when \( \mu = 0 \) requires prior
knowledge of the minimum step-size;
while this is straightforward for \( L \)-smooth functions, it is not clear how
to extend such result to non-strongly convex acceleration with locally
Lipschitz gradients.
For example, while \citet{li2023generalized} show that the \( (r, l)
\)-smoothness (a valid directional smoothness function) is bounded over the
iterate trajectory, their rate does not adapt to the optimization path.

%% file: figures/quadratic_bound.tex

\begin{tikzpicture}[
      declare function={
        objective(\x)=  (\x<=0) * pow(-\x, 3) / 2 + (\x>0) * pow(x, 2);
        upperBound(\x)=     6 * pow(\x + 2, 2) - 6 * (\x + 2) + 4; 
        locaUpperBound(\x)=     5 * pow(\x + 1.9, 2) - 6 * (\x + 1.9) + 3.5; 
      }
    ]
    \begin{axis}[
      width=0.9\textwidth,
      height=6cm,
      axis x line=none, axis y line=none,
      ymin=-3.25, ymax=6, ytick={-5,...,5}, ylabel=$y$,
      xmin=-3.5, xmax=1.5, xtick={-4,...,2}, xlabel=$x$,
    ]

    \addplot[name path=function, domain=-3.5:5.23, samples=200, objective, line width=2pt]{objective(x)};
    \addplot[name path=upperBound, domain=-3.5:7, samples=200, oracle, line width=2pt]{upperBound(x)};
    \addplot[name path=localUpperBound, domain=-3.5:7, samples=200, bound, line width=2pt]{locaUpperBound(x)};

    \node[label={195:$\xk$},circle,fill,inner sep=2pt] at (axis cs:-1.9, 3.5) {};
    \node[label={195:$\xkk$},circle,fill,inner sep=2pt] at (axis cs:-1.5, 1.6875) {};
    \node[label={90:$\xopt$},circle,fill,inner sep=2pt] at (axis cs:0,0) {};
    
    \node[label={180:$f(x)$}] at (axis cs:1.5,0.4) {};

    \draw [line width=2pt, style=dashed, color=black] (-2.4,1.675) -- (-1.5, 1.675);
    \draw [line width=2pt, style=dashed] (-2.4,3.5) -- (-1.9, 3.5);
    \draw [decorate, decoration={brace, amplitude=4pt}, line width=2pt] (-2.4,1.675) -- (-2.4,3.5) node [midway, anchor=east, xshift=-1mm, outer sep=1pt,font=\small]{Actual Progress};

    \draw [line width=2pt, style=dashed, color=black] (-0.5,2.4) -- (-1.5, 2.4);
    \draw [line width=2pt, style=dashed] (-1.1,3.5) -- (-0.5, 3.5);
    \draw [decorate, decoration={brace, amplitude=4pt, mirror}, line width=2pt] (-0.5, 2.4) -- (-0.5,3.5) node [midway, anchor=west, xshift=1mm, outer sep=1pt,font=\small]{\( L \)-Smooth};

    \draw [line width=2pt, style=dashed, color=purple] (0.5,1.9) -- (-1.06, 1.9);
    \draw [line width=2pt, style=dashed, color=purple] (-0.4,3.5) -- (0.5, 3.5);
    \draw [decorate, decoration={brace, amplitude=4pt, mirror}, line width=2pt] (0.5, 1.9) -- (0.5, 3.5) node [midway, anchor=west, xshift=1mm, outer sep=1pt,font=\small]{\( M_k \)-Smooth};
    \end{axis}
\end{tikzpicture} 

%% file: sections/adaptive_learning_rates.tex
Converting our sub-optimality bounds into convergence
rates requires adapted step-sizes satisfying \( \etak < 2 / M( \xk, \xkk)\).
Given an adapted step-size, the directional descent lemma
(\cref{eq:local-descent-lemma}) implies GD decreases \( f \) and
we can obtain fast rates if the step-sizes are bounded below.
However, \( \xkk \) is itself a function of \( \etak \), meaning adapted
step-sizes are not straightforward to compute.

For $L$-smooth \( f \), the different directional smoothness functions \( M \)
introduced in \cref{sec:local-direct-smoothn} satisfy
\( M( \xk, \xkk) \leq 2L \).
This implies \( \etak < \frac{1}{L} \) is trivially adapted.
As such step-sizes don't capture local properties of \( f \),
we introduce the notion of \emph{strongly adapted step-sizes}, which satisfy
\begin{align}\label{eq:strongly-adapted-step-size}
    \eta_k = 1 / M(x_{k+1}(\etak), x_k).
\end{align}
\cref{eq:local-descent-lemma} implies GD with a strongly adapted step-size
makes guaranteed progress as,
\begin{align}
    \label{eq:ad-4}
    f(x_{k+1})
    \leq f(x_k) - \sbr{2M(x_{k+1}, x_k)}^{-1} \sqn{\nabla f(x_k)}_2.
\end{align}
This progress is greater than that guaranteed by \( L \)-smoothness
when \( M( \xk, \xkk) < L \) and holds even when \( f \) is not \( L \)-smooth.
However, it is not clear a priori if (i) strongly adapted step-sizes exist or if (ii)
any iterative method achieves the progress in Eq.~\eqref{eq:strongly-adapted-step-size}.
Surprisingly, we provide a positive answer to both questions.
Strongly adapted \( \etak \) are computable and we also prove GD with
the Polyak step-size adapts to any choice of directional smoothness, including
the optimal point-wise smoothness.
Before presenting this strong result, we consider the illustrative case of
quadratic minimization.


\subsection{Adaptivity in Quadratics}%
\label{sec:quadratic-case}

Now we show that step-sizes adapted to both the point-wise smoothness $M$ and
the path-wise smoothness $A$ exist when \( f \) is quadratic.
Let \( f(x) = x^{\top} B x / 2 - c^{\top} x, \) where $B$ is positive
semi-definite.
Assuming \( \cbr{\etak} \) is strongly adapted to the directional smoothness,
\Cref{eq:1} implies
\begin{equation}
    f(\bar{x}_k) - f(\xopt) \leq \frac{\sqn{x_0 - \xopt}_2}{2 \sum_{i=0}^k \eta_i}
    = \frac{\sqn{x_0 - \xopt}_2}{2 \sum_{i=0}^k \frac{1}{M (x_i, x_{i+1})}}
    \label{eq:quad-1}
    \leq \frac{\sqn{x_0 - \xopt}_2}{2 (k+1)} \frac{\sum_{i=0}^k M (x_i, x_{i+1})}{k+1},
\end{equation}
where we used $\eta_i M_i = 1$ as well as Jensen's inequality.
This guarantee depends solely on the average directional smoothness along the
optimization trajectory \( \cbr{x_0, x_1, \ldots} \).
When \( f \) is quadratic, we can exactly compute these smoothness
constants.
In particular, the point-wise directional smoothness is,
\[
    D(x_i, x_{i+1}) = 2 \norm{B \nabla f(x_i)}_2 / \norm{\nabla f(x_i)}_2.
\]
Notably, \( D(x_i, x_{i+1}) \) has no dependence on \( x_{i+1} \) and the
corresponding strongly adapted step-size is given by
\( \eta_i = \norm{\nabla f(x_i)}_2 / (2 \norm{B \nabla f(x_i)}_2) \) --- see
\Cref{lemma:quadratic-pointwise-direction-smoothness}.
Remarkably, this expression recovers the step-size proposed by
\citet{dai2006computational}, who show it approximates the Cauchy step-size
and converges to the ``edge-of-stability'' \citep{cohen2021stability}
at \( 2 / L \) as \( k \rightarrow \infty \).
Combining this simple expression with \cref{eq:quad-1} gives a fast,
non-asymptotic convergence rate for GD and new theoretical justification for
their work.


We can also compute the path-wise directional smoothness in closed form.
As \Cref{lemma:quadratic-pathwise-direction-smoothness} shows,
\[
    A(x_i, x_{i+1}) = \nabla f(x_i)^{\top} B \nabla f(x_i)
    / \nabla f(x)^{\top} \nabla f(x),
\]
and $\eta_i = {\nabla f(x_i)^{\top} \nabla f(x_i)} / [\nabla f(x_i)^{\top} B \nabla f(x_i)]$
is the well-known Cauchy step-size.
Path-wise directional smoothness thus provides another interpretation (and
convergence guarantee) for the Cauchy step-size, which is traditionally
derived by minimizing $f(x-\eta \nabla f(x))$ in $\eta$.

\begin{figure*}[t]
    \centering
    \includegraphics[width=0.98\textwidth]{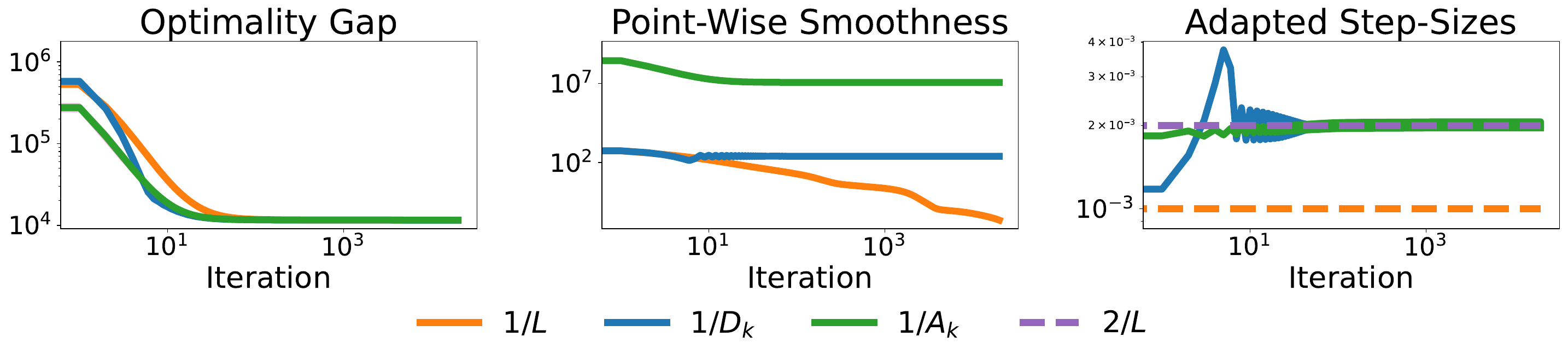}
    \caption{Performance of GD with different step-size rules
        for a synthetic quadratic problem.
        We run GD for 20,000 steps on 20 random quadratic problems with $L=1000$ and
        Hessian skew.
        Left-to-right, the first plot shows the optimality gap \( f(\xk) - f(\xopt)
        \), the second shows the point-wise directional smoothness \( D( \xk, \xkk) \),
        and the third shows step-sizes used by the different methods.
    }
    \label{fig:quadratics}
\end{figure*}


\subsection{Adaptivity for Convex Functions}%
\label{sec:adaptivity-general-case}

In the last subsection, we proved that strongly adapted step-sizes for the
point-wise and path-wise directional smoothness functions have closed-form
expressions when \( f \) is quadratic.
Moreover, these step-sizes recover two classical schemes from the optimization
literature, giving them new justification and fast convergence rates.
Now we consider the existence of strongly adapted step-sizes for general
convex functions.
Our first result gives simple conditions for
\cref{eq:strongly-adapted-step-size} to have at least one solution when \( M
\) is the point-wise directional smoothness.
\begin{restatable}{proposition}{propExistenceD}\label{prop:existence-D}
    If \( f \) is convex and continuously differentiable, then either (i) \( f \)
    is minimized along the ray \( x(\eta) = x - \eta \grad(x) \) or (ii) there
    exists \( \eta > 0 \) satisfying $\eta = 1/D(x, x-\eta \nabla f(x))$.
\end{restatable}
The next proposition uses a similar argument with slightly stronger conditions
to show existence of strongly adapted step-sizes for the path-wise smoothness.
\begin{restatable}{proposition}{propExistenceA}\label{prop:existence-A}
    If \( f \) is convex and twice continuously differentiable, then either (i)
    \( f \) is minimized along the ray \( x(\eta) = x - \eta \grad(x) \) or (ii)
    there exists \( \eta > 0 \) satisfying $\eta = 1/A(x, x-\eta \nabla f(x))$.
\end{restatable}
Propositions~\ref{prop:existence-D} and~\ref{prop:existence-A} do not assume
the global smoothness;
although neither proof is constructive, it is possible to compute strongly
adapted step-sizes for the point-wise directional smoothness using root-finding
methods.
We show in \cref{sec:experiments} that if \( f \) is twice differentiable, then
strongly adapted step-sizes can be found via Newton's method using only
Hessian-vector products, \( \nabla^2 f(x) \nabla f(x) \).

\subsubsection{Exponential Search}

Now we show that the exponential search algorithm developed by
\citet{yair2022parameter} can be used to find step-sizes that adapt
\emph{on average} to the directional smoothness.
Consider a fixed optimization horizon $k$ and denote by $x_i (\eta)$ the
sequence of iterates obtained by running GD from $x_0$ using a fixed step-size
$\eta$.
Define the criterion function,
\begin{align}
    \label{eq:psi}
    \psi (\eta)
    = \frac{\sum_{i=0}^k \norm{\nabla f(x_i (\eta))}_2^2}{\sum_{i=0}^k M(x_i (\eta), x_{i+1} (\eta))
        \norm{\nabla f(x_i (\eta))}_2^2},
\end{align}
and suppose that we have a step-size $\eta$ that satisfies
$\psi(\eta)/ 2 \leq \eta \leq \psi (\eta)$.
Using these bounds in \Cref{prop:convex-any-stepsize} yields
the following convergence rate,
\begin{align}
    \label{eq:exp-5} \textstyle
    f(\bar{x}_k) - f_{*} \leq \left[ \frac{k \sum_{i=0}^k M(x_i, x_{i+1}) \norm{\nabla f(x_i)}_2^2}{\sum_{i=0}^k \norm{\nabla f(x_i)}_2^2} \right] \norm{x_0 - x^{*}}_2^2.
\end{align}
While $\eta$ does not adapt to each directional smoothness $M(x_i, x_{i+1})$
along the path, it adapts to a \emph{weighted average} of the directional
smoothness constants, where the weights are the observed squared gradient
norms.
This is always smaller than the maximum directional smoothness along the
trajectory and can be much smaller than the global smoothness.
Furthermore, we have reduced our problem to finding $\eta \in [\psi(\eta)/2,
\psi(\eta)]$, which is similar to the problem \citet{yair2022parameter} solve
with exponential search.
We adopt their approach as \Cref{alg:exponential-search-gd} and give a
convergence guarantee.

\begin{restatable}{theorem}{exponentialsearch}\label{thm:exponential-search-gd}
    Assume $f$ is convex and $L$-smooth.
    Then Algorithm~\ref{alg:exponential-search-gd} with $\eta_0 > 0$
    requires at most \( 2 K (\log \log(2\eta_0/L) \lor 1) \) iterations of GD
    and in the last run it outputs a step-size $\eta$ and point
    $\bar{x}_K = \frac{1}{K} \sum_{i=0}^{K-1} x_i (\eta)$ such that exactly one of the
    following holds: 
    \begin{align*}
         & \text{Case 1: } \quad \eta = \eta_0 \quad
        \text{ and } \quad
        f(\bar{x}_K) - f(\xopt)
        \leq \frac{\norm{x_0 - x^{*}}_2^2}{2K \eta_0}   \\
         & \text{Case 2: } \quad \eta \neq \eta_0 \quad
        \text{ and } \quad
        f(\bar{x}_K) - f^{*}
        \leq \frac{\norm{x_0 - x^{*}}_2^2}{2K}
        \left[ \frac{\sum_{i=0}^k M_i \norm{\nabla f(x_i^{\prime})}_2^2}
        {\sum_{i=0}^k \norm{\nabla f(x_i^{\prime})}_2^2} \right],
    \end{align*}
    where $M_i \eqdef M(x_i^{\prime}, x_{i+1}^{\prime})$ and
    $x_i^{\prime}$ are the iterates generated by GD
    with step-size $\eta^{\prime} \in \left[\eta, 2 \eta \right]$.
\end{restatable}

\Cref{thm:exponential-search-gd} requires \( f \) to be $L$-smooth, but has
only a $\log \log$ dependence on the global smoothness constant.
Moreover, the rate scales with the weighted average of
smoothness constants along a very close trajectory
$\cbr{x_1^{\prime}, x_2^{\prime}, \ldots}$.
In the next section, we give convergence bounds that depend on the unweighted
average of the directional smoothness constants along the \emph{actual}
optimization trajectory.

\subsubsection{Polyak's Step-Size Rule}

Our theory so-far suggests using strongly adapted step-sizes, but
neither root-finding nor exponential search are practical methods
for large-scale optimization.
Thus, we now consider other step-size selection rules which may leverage
directional smoothness.
In particular, the Polyak step-size sets,
\begin{equation}
    \label{eq:polyak-step-size}
    \etak = \gamma \rbr{f(\xk) - f(\xopt)} / \sqn{\nabla f(\xk)}_2,
\end{equation}
for some \( \gamma > 0 \), which is optimal for smooth and non-smooth
optimization \citep{hazan2019revisiting} given knowledge of \( f(\wopt) \).
Surprisingly, we show that GD with the Polyak step-size also achieves
the same guarantee as strongly adapted step-sizes without
knowledge of the directional smoothness.

\begin{restatable}{theorem}{polyak}\label{thm:polyak}
    Suppose that $f$ is convex and differentiable and let $M$ be any
    directional smoothness function for $f$. Let $\Delta_0 := \|x_0 -\xopt\|_2^2$.
    Then GD with the Polyak step-size and \( \gamma \in \rbr{1, 2} \) satisfies
    \begin{align}
        \label{eq:polyak-guarantee-1}
        f(\bar x_k) - f(\xopt) \leq
        \frac{c(\gamma) \Delta_0}
        {2\sum_{i=0}^{k-1} M(x_i, x_{i+1})^{-1}},
    \end{align}
    where \( c(\gamma) = \gamma / (2 - \gamma)(\gamma - 1) \) and
    \(
    \bar x_k = \sum_{i=0}^{k-1} \sbr{M(x_i, x_{i+1})^{-1} x_i} /
    \rbr{\sum_{i=0}^{k-1} M(x_i, x_{i+1})^{-1}} \).
\end{restatable}
\cref{thm:polyak} measures sub-optimality at an average iterate
obtained using the directional smoothness.
However, it also holds for the best iterate,
\( \hat x_k = \argmin_{i \in [k]} f(\x_i) \),
meaning no knowledge of the directional smoothness is required to obtain the
guarantee.
We prove an alternative guarantee for the Polyak step-size in
\cref{thm:polyak-alternate}, where the progress depends on the sum of
step-sizes rather than on the average directional smoothness.
This shows that the step-size in \cref{eq:polyak-step-size} can itself be
viewed as a measure of local smoothness, albeit without formal justification.

Compared with the standard guarantee for the Polyak
step-size under $L$-smoothness,
\( f(\bar{x}_k) - f(\xopt) \leq 2 L \Delta_0 / k \) \citep{hazan2019revisiting},
our analysis in \Cref{thm:polyak} with the choice $\gamma = 1.5$ yields
\begin{equation*}
    f(\bar{x}_k) - f(\xopt)
    \leq \frac{3 \Delta_0} {\sum_{i=0}^{k-1} M(x_i, x_{i+1})^{-1}}
    \leq \frac{3 \Delta_0}{k} \frac{\sum_{k=0}^{k-1} M(x_i, x_{i+1})}{k},
\end{equation*}
where the second bound follows from Jensen's inequality
and shows that the convergence depends on the average
directional smoothness along the trajectory, rather than on $L$.
If \( f \) is \( L \)-smooth, then $M(\xk, \xkk) \leq L$ immediately
recovers the classic rate for Polyak's method up to a \( 3/2 \) constant factor.
If \( f \) is not \( L \)-smooth, but \( M(\xk, \xkk) \) is bounded, then
\cref{eq:polyak-guarantee-1} generalizes the \( O(1/k) \) rate proved
concurrently by \citet{takezawa2024polyak}, but for any choice of directional
smoothness (of which \( (L_0, L_1) \)-smoothness \citep{zhang2020clipping} is
but one).


\textbf{Comparison with strongly adapted step-sizes.} As we saw for
quadratics, strongly adapted step-sizes for any directional smoothness function
allow us to obtain the following convergence rate,
\begin{equation*}
    f(\bar{x}_k) - f(\xopt)
    \leq \frac{\sqn{x_0 - \xopt}_2}{2 \sum_{i=0}^{k-1} M (x_i, x_{i+1})^{-1}}.
\end{equation*}
This is matches the guarantee given by \cref{eq:polyak-guarantee-1} up to
constant factors.
As a result, we give a positive answer to the question posed earlier in this
section: GD with the Polyak step-size achieves the same convergence for any
smoothness function $M$ as GD with step-sizes strongly adapted to $M$.

\textbf{Application to the optimal directional smoothness.}
\Cref{thm:polyak} holds for \emph{every} directional smoothness
function $M$.
Therefore we can specialize \cref{eq:polyak-guarantee-1} with the optimal
point-wise directional smoothness $H$ (as defined in
\cref{eq:directional-smoothness}) and $\gamma = 1.5$ to get the guarantee,
\begin{equation}
    \label{eq:polyak-optimal}
    \min_{i \in [k-1]} \left[ f(x_i) - f(\xopt) \right]
    \leq \frac{3 \sqn{x_0 - \xopt}_2}{\sum_{i=0}^{k-1} H(x_i, x_{i+1})^{-1}}.
\end{equation}
This rate requires computing the iterate with the minimum function value, but
that is easy to track during optimization.
Unlike our previous results, \cref{eq:polyak-optimal} requires no access to
the optimal point-wise smoothness, yet obtains a dependence on the tightest
constant possible.

\subsection{Normalized Gradient Descent}

Now we change directions slightly and study normalized GD, whose
convergence also depends on the directional smoothness.
Normalized GD uses step-sizes which are divided by the gradient
magnitude,
\begin{equation}
    \label{eq:normalized-gd}
    \xkk = \xk - \frac{\etak}{\norm{\grad(\xk)}_2}\grad(\xk).
\end{equation}
Our next theorem shows that normalized GD obtains a guarantee which depends
solely on the average of the point-wise directional smoothness
$D_k := D(x_k, x_{k+1})$ despite no explicit knowledge of \( D_k \).

\begin{restatable}{theorem}{ngd}\label{thm:ngd}
    Suppose that $f$ is convex and differentiable.
    Let $D$ be the point-wise directional smoothness defined by
    \cref{eq:directional-smoothness} and $\Delta_0 := \|x_0 -\xopt\|_2^2$.
    Then normalized GD with a sequence of non-increasing step-sizes \( \etak \)
    satisfies
    \begin{equation}
        \label{eq:ngd-guarantee}
        f(\hat x_k)-f(\xopt)
        \le \frac{\Delta_0 + \sum_{i=0}^{k-1}\eta_i^2}{2k^2}\left( \frac{f(x_0)}{\eta_0^2} - \frac{f(\xopt)}{\eta_{k-1}^2} \right)
        + \frac{\Delta_0 + \sum_{i=0}^{k-1}\eta_i^2}{2k}\sum_{i=0}^{k-1}   \frac{M(x_i, x_{i+1})}{k}   ,
    \end{equation}
    where $\hat x_k = \arg\min_{i \in [k-1]} f(x_i)$.
    If \( \max_{i \in [k-1]} M(x_i, x_{i+1}) \) is bounded for all \( k \)
    (i.e.~\( f \) is \( L \)-smooth), then for $\eta_i = 1/\sqrt{i}$ we have
    $f(\hat x_k)-f(\xopt) \in \mathcal{O}(1/k)$ and for $\eta_i = 1/\sqrt{i}$
    we get the anytime result
    $f(\hat x_k)-f(\xopt) \in \mathcal{O}(\log(k)/k)$.
\end{restatable}

%

\begin{figure*}[t]
    \centering
    \includegraphics[width=0.98\textwidth]{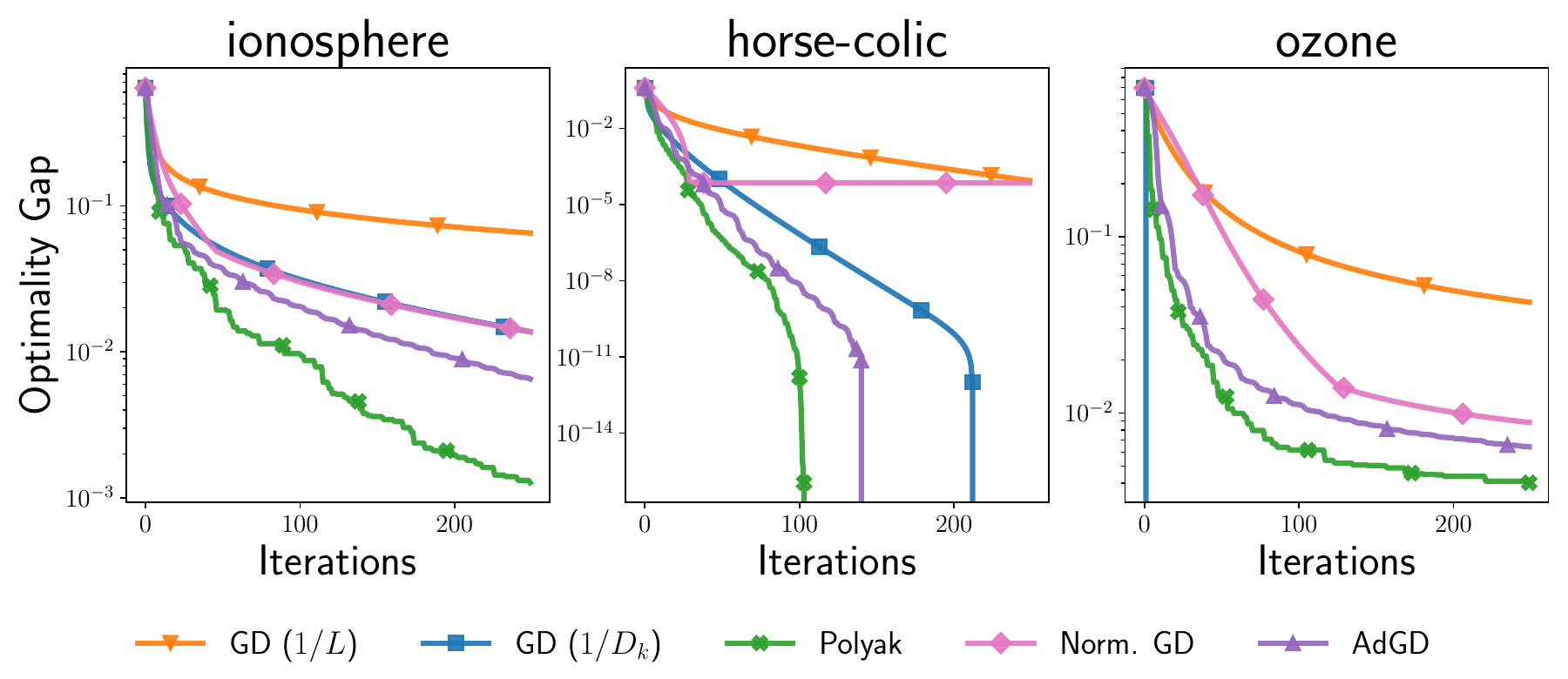}
    \caption{
        Comparison of GD with \( \etak = 1 / L \),
        step-sizes strongly adapted to the point-wise smoothness (\(
        \etak = 1/D( \xk, \xkk) \)), and the Polyak
        step-size against normalized GD (Norm.  GD) and the AdGD method on
        three logistic regression problems.
        AdGD uses a smoothed version of the point-wise directional smoothness from the
        previous iteration to set \( \etak \).
        We find that GD methods with adaptive step-sizes consistently outperform GD
        with \( \etak = 1 / L \) and even obtain a linear rate on \texttt{horse-colic}.
    }
    \label{fig:logistic-comparison}
\end{figure*}

\Cref{thm:ngd} gives a rate for normalized GD which is valid for any convex $f$
without any dependence on global smoothness.
However, does not adapt to any smoothness function like the Polyak step-size.

%% file: sections/experiments.tex


We evaluate the practical improvement of our convergence rates over those
using \( L \)-smoothness on two logistic regression problems taken from
the UCI repository \citep{asuncion2007uci}.
\cref{fig:rate-comparison} compares GD with strongly adapted step-sizes
\( \eta = 1 / M_k \), where \( M_k \) is the point-wise
smoothness, against GD with the Polyak step-size.
We also plot the exact convergence rates for each method, \cref{eq:1} and
\cref{eq:polyak-guarantee-1}, respectively, and compare against the classical
guarantee for both methods.
Our convergence rates are an order of magnitude tighter on the
\texttt{ionosphere} dataset and display a remarkable ability to adapt to the
path of optimization on \texttt{mammographic}.

Figure~\ref{fig:quadratics} compares the performance of GD with strongly
adapted step-sizes and with the fixed step-size $\etak = 1/L$ for a synthetic
quadratic with Hessian skew \citep{pan21_eigen}.
Results are averaged over twenty random problems.
We find that strongly adapted step-sizes lead to significantly faster
convergence.
Since \( A_k, D_k \ll L \), the adapted step-sizes are larger than $2/L$,
especially at the start of training;
they eventually converge to $2/L$, indicating these methods
operate at the edge-of-stability
\citep{cohen2021stability,cohen22_adapt_gradien_method_at_edge_stabil}.
This is consistent with \citet{ahn2022unstable, pan2023toward}, who show local
smoothness is correlated with edge-of-stability behavior.

We conclude with a comparison of empirical convergence rates on three
additional logistic regression problems from the UCI repository.
We compare GD with \( \etak = 1 / L \), GD with step-sizes strongly adapted to
the point-wise smoothness (\( \etak = 1/D_k \)), GD with the Polyak step-size
(Polyak), and normalized GD (Norm.
GD) against the AdGD method \citep{malitsky2020descent}.
The Polyak step-size performs best on every dataset but \texttt{ozone}, where
GD with \( \etak = 1/D_k \) solves the problem to high accuracy in just a few
iterations.
Thus, although Polyak step-sizes have the optimal dependence on directional
smoothness, computing strongly adapted step-sizes can still be advantageous.

%% file: sections/conclusion.tex

We present new sub-optimality bounds for GD under novel measures of local
gradient variation which we call directional smoothness functions.
Our results hold for any step-sizes, improve over standard
analyses when \( \etak \) is adapted to the choice of directional smoothness, 
and depend only on properties of \( f \) local to the optimization path.
For convex quadratics, we show that computing step-sizes strongly
adapted to directional smoothness functions is straightforward and recovers
two well-known step-size schemes, including the Cauchy step-size.
In the general case, we prove that an algorithm based on exponential search
gives a weighted-version of the path-dependent convergence rate with no need
for adapted step-sizes.
We also show that GD with the Polyak step-size and normalized GD both obtain
fast rates with no dependence on the global smoothness parameter.
Crucially, the Polyak step-size adapts to any choice of directional
smoothness, including the tightest possible parameter.


%% file: appendices/directional_smoothness_proofs.tex

\directionalSmoothness*
\begin{proof}
    By the convexity of $f$ we have
    \begin{align*}
        f(x) + \ev{ \nabla f(x) , y-x } \leq f(y).
    \end{align*}
    Rearranging and then using Cauchy-Schwarz we get
    \begin{align*}
        f(x) & \leq f(y) + \ev{ \nabla f(x) , x-y }                                               \\
             & = f(y) + \ev{ \nabla f(y) , x-y } + \ev{ \nabla f(x) - \nabla f(y) , x-y }         \\
             & \leq f(y) + \ev{ \nabla f(y) , x-y } + \norm{\nabla f(x) - \nabla f(y)} \norm{x-y} \\
             & = f(y) + \ev{ \nabla f(y) , x-y } + \frac{D(x, y)}{2} \norm{x-y}^2.
        \qedhere
    \end{align*}
\end{proof}

\pathSmoothness*
\begin{proof}
    Starting from the fundamental theorem of calculus,
    \begin{align*}
        f(y) - f(x) - \abr{\grad(x), y - x}
         & = \int_{0}^1 \abr{\grad(x + t(y - x)) - \grad(x), y - x} dt \\
         & \leq \int_{0}^1 A(x, y) t \norm{x - y}_2^2 dt               \\
         & = \frac{A(x, y)}{2} \norm{y - x}^2_2.
    \end{align*}
    which completes the proof.
\end{proof}


\dirSmoothTight*

\begin{proof}
    Let $H_f$ denote the optimal pointwise directional smoothness associated with
    some convex and differentiable function $f: \mathbb{R}^d \rightarrow
        \mathbb{R}$ (as defined in \cref{eq:directional-smoothness}), and $D_f$ denote
    the pointwise directional smoothness associated with $f$ (as defined in
    \cref{eq:directional-smoothness}).
    For any $t$, the statement of \eqref{eq:dir-smoothness-tight-prop} is equivalent
    to saying $H_f > t \frac{\norm{\nabla f(x) - \nabla f(y)}}{\norm{x-y}}$ for
    all $x,y \in \mathbb{R}^d$ and convex, differentiable $f$.
    Observe that
    \Cref{lemma:directional-smoothness} already shows that for all convex and
    differentiable functions $f: \mathbb{R}^d \rightarrow \mathbb{R}$
    \begin{align*}
        H_f(x, y) \leq D_f(x, y) = 2 \frac{\norm{\nabla f (x) - \nabla f
                (y)}}{\norm{x-y}}
    \end{align*}
    for all $x, y \in \mathbb{R}^d$.
    In order to show that this is tight, we suppose
    by the way of contradiction that there exists some $2> t \geq 0$ such that
    for all convex and differentiable functions $f: \mathbb{R}^d \rightarrow \mathbb{R}$
    \begin{align}
        \label{eq:n0}
        H_f(x, y) \leq t \cdot \frac{\norm{\nabla f (x) - \nabla f (y)}}{\norm{x-y}}
    \end{align}
    for all $x, y \in \mathbb{R}$.
    We shall show that no such $t$ exists by showing for each such $t$ there
    exists a function $f_t$ such that \cref{eq:n0} does not hold.

    Consider $f_{\epsilon}(x) = \sqrt{x^2 + \epsilon^2}$ for $\epsilon \leq 1$.
    The function $f$ is differentiable.
    Moreover
    \begin{align*}
        f_{\epsilon}^{\prime} (x) = \frac{x}{\sqrt{x^2 + \epsilon^2}}, &  & f_{\epsilon}^{\prime\prime} (x) = \frac{\epsilon^2}{(\epsilon^2 + x^2)^{\frac{3}{2}}} \geq 0.
    \end{align*}
    Therefore $f$ is convex.
    Let $g(x) = \abs{x}$.
    Fix $x=1$ and $y = 0$, we have
    \begin{align*}
        \abs{g(x) - g(y) - {\sign(y) \cdot (x-y)}}
         & \leq \abs{f_{\epsilon}(x) - f_{\epsilon}(y) - {f_{\epsilon}^{\prime}(y) \cdot (x-y)}}
        + \abs{g(x)-f_{\epsilon}(x)}                                                                                        \\
         & \hspace{3em} + \abs{g(y)-f_{\epsilon}(y)} + \abs{{(f_{\epsilon}^{\prime}(y)-\sign(y)) \cdot (x-y)}}              \\
         & = \abs{f_{\epsilon}(x) - f_{\epsilon}(y) - {f_{\epsilon}^{\prime}(y) \cdot (x-y)}} + \abs{1-\sqrt{1+\epsilon^2}} \\
         & \hspace{3em} + \abs{0-\sqrt{\epsilon^2}} + \abs{(0 - 0) \cdot (1-0)}                                             \\
         & \leq \abs{f_{\epsilon}(x) - f_{\epsilon}(y) - {f_{\epsilon}^{\prime}(y) \cdot (x-y)}} + 2\epsilon.
    \end{align*}
    Now observe that
    \begin{align*}
        g(x) - g(y) - {\sign(y) \cdot (x-y)} & = \abs{1} - \abs{0} - 0 \cdot (1-0) = 1.
    \end{align*}
    Therefore
    \begin{align}
        \label{eq:n5}
        \abs{f_{\epsilon}(x) - f_{\epsilon}(y) - {f_{\epsilon}^{\prime}(y) \cdot (x-y)}} \geq 1 - 2\epsilon.
    \end{align}
    By definition we have $\frac{1}{2} \sqn{x-y} = \frac{1}{2}$, therefore
    \begin{align}
        \label{eq:n6}
        H_f(x, y) \; = \; \frac{\abs{f_{\epsilon}(x) - f_{\epsilon}(y) -
        {f_{\epsilon}^{\prime}(y) \cdot (x-y)}}}{\frac{1}{2} \sqn{x-y}} \geq 2 -
        4\epsilon.
    \end{align}
    But by our starting assumption we have that there exists some $t < 2$ such
    that $H_f(x, y) \leq t \frac{\norm{f'(x)- f^{\prime}(y)}}{\abs{x-y}}$ for all
    differentiable and convex functions $f$.
    Applying this to $f=f_{\epsilon}$ we get
    \begin{align}
        \label{eq:n7}
        H_{f_{\epsilon}}(1, 0) & \leq t \frac{\abs{f_{\epsilon}^{\prime}(1) -
                f_{\epsilon}^{\prime}(0)}}{\abs{1}} = t \cdot \frac{1}{\sqrt{1 + \epsilon^2}}
        \leq t.
    \end{align}
    Combining \Cref{eq:n6,eq:n7} we have
    \begin{align*}
        2 - 4\epsilon \leq H_{f_{\epsilon}}(1, 0) \leq t
    \end{align*}
    Rearranging we get
    \begin{align*}
        2-t \leq 4\epsilon
    \end{align*}
    Choosing $\epsilon = \frac{2-t}{8} > 0$ we get a contradiction.
    It follows that the minimal $t$ such that $H(x, y) \leq t
        \frac{\abs{f^{\prime} (x) - f^{\prime}(y)}}{\abs{x-y}}$ for all convex and
    differentiable $f$ is $t=2$.
\end{proof}

\begin{restatable}{lemma}{localDescentLemma}\label{lemma:local-descent-lemma}
    One step of gradient descent with step-size \( \etak > 0 \) makes progress
    as
    \begin{equation*}
        f(\xkk) \leq f(\xk) - \etak \rbr{1 - \frac{\etak M(\xk, \xkk)}{2}} \norm{\grad(\xk)}_2^2.
    \end{equation*}
\end{restatable}
\begin{proof}
    Starting from \cref{eq:directional-smoothness}, we have
    \begin{align*}
        f(\xkk)
         & \leq f(\xk) + \abr{\grad(\xk), \xkk - \xk} + \frac{M(\xk, \xkk)}{2} \norm{\xkk - \xk}_2^2     \\
         & = f(\xk) - \etak \norm{\grad(\xk)}_2^2 + \frac{\etak^2 M(\xk, \xkk)}{2} \norm{\grad(\xk)}_2^2 \\
         & = f(\xk) - \etak\rbr{\frac{1 - \etak M(\xk, \xkk)}{2}}\norm{\grad(\xk)}_2^2.
    \end{align*}
\end{proof}

%% file: appendices/path_dependent_rates_proofs.tex

\begin{restatable}{lemma}{directionalSC}\label{lemma:directional-sc}
    If \( f \) is convex, then for any \( x, y \in \R^d \),
    \begin{equation}
        f(y) \geq f(x) + \abr{\grad(x), y - x} + \frac{\mu(x, y)}{2}\norm{y - x}_2^2.
    \end{equation}
    If \( f \) is \( \mu \) strongly convex, then \( \mu(x, y) \geq \mu \).
\end{restatable}
\begin{proof}
    The fundamental theorem of calculus implies
    \begin{align*}
        f(x) - \abr{\grad(x), y - x}
         & = \int_{0}^1 \abr{\grad(x + t(y - x)) - \grad(x), y - x} dt \\
         & \geq \int_{0}^1 \mu(x, y) t \norm{x - y}_2^2 dt             \\
         & = \frac{\mu(x, y)}{2} \norm{y - x}^2_2.
    \end{align*}
    Note that we have implicitly used convexity to verify the inequality in the
    second line in the case where \( \mu(x, y) = 0 \).
    Now assume that \( f \) is \( \mu \) strongly convex.
    As a standard consequence of strong-convexity, we obtain:
    \begin{align*}
        \frac{\abr{\grad(x + t(y - x)) - \grad(x), y - x} }
        {t\norm{x-y}_2^2}
         & =\frac{\abr{\grad(x + t(y - x)) - \grad(x), x + t (y - x) - x} }
        {t^2 \norm{x-y}_2^2}                                                 \\
         & \geq \mu \frac{\norm{x - t(y - x) - x}_2^2}{t^2 \norm{y - x}_2^2} \\
         & = \mu.
    \end{align*}
\end{proof}

\scSplitAnalysis*
\begin{proof}
    First note that \( \lambda_i < 0 \) for \( i \in \calG \) and \( \lambda_i
    \geq 0 \) for \( i \in \calB \).
    We start from \cref{eq:local-descent-lemma},
    \begin{align*}
        f(\xkk)
         & \leq f(\xk) + \etak\rbr{\frac{\etak M(\xk, \xkk)}{2} - 1}\norm{\grad(\xk)}_2^2                   \\
         & = f(\xk)
        + \mathbbm{1}_{k \in \calG} \cdot \sbr{\frac{\etak\lambda_k}{2} \norm{\grad(\xk)}_2^2 }
        + \mathbbm{1}_{k \in \calB} \cdot \sbr{\frac{\etak\lambda_k}{2} \norm{\grad(\xk)}_2^2}              \\
         & \leq f(\xk) + \mathbbm{1}_{k \in \calG} \cdot \sbr{\etak\lambda_k \mu_k \rbr{f(\xk) - f(\xopt)}}
        + \mathbbm{1}_{k \in \calB} \cdot \sbr{\frac{\etak\lambda_k}{2} \norm{\grad(\xk)}_2^2},
    \end{align*}
    where we used that directional strong convexity gives
    \[ \norm{\grad(\xk)}_2^2 \geq  2\mu_k \rbr{f(\xk) - f(\xopt)}. \]
    Subtracting \( f(\xopt) \) from both sides and then recursively applying the
    inequality gives the result.
\end{proof}

\scIteratesAnalysis*
\begin{proof}
    Let \( \Delta_k = \norm{\xk - \xopt}_2^2 \) and observe \[ \Delta_k =
        \norm{\xk - \xkk + \xkk - \xopt}_2^2 = \Delta_{k+1} + \norm{\xk - \xkk}_2^2 +
        2\abr{\xk - \xkk, \xkk - \xopt}.
    \]
    Using this expansion in \( \Delta_{k+1} - \Delta_k \), we obtain
    \begin{align*}
        \Delta_{k+1} - \Delta_k
         & = -\norm{\xk - \xkk}_2^2
        -2 \abr{\xk - \xkk, \xkk - \xopt}                                                     \\
         & = - \etak^2 \norm{\grad(\xk)}_2^2
        - 2 \etak \abr{\grad(\xk), \xkk - \xopt}                                              \\
         & = - \etak^2 \norm{\grad(\xk)}_2^2
        - 2 \etak \abr{\grad(\xk), \xkk - \xk}
        - 2 \etak \abr{\grad(\xk), \xk - \xopt}.
        \intertext{%
            Now we control the inner-products with directional strong convexity and
            directional smoothness.
        }
         & \leq - \etak^2 \norm{\grad(\xk)}_2^2
        - 2 \etak \abr{\grad(\xk), \xkk - \xk}
        + 2 \etak \sbr{f(\xopt) - f(\xk) - \frac{\mu_k}{2}\Delta_k}                           \\
         & \leq - \etak^2 \norm{\grad(\xk)}_2^2
        + 2 \etak \sbr{f(\xk) - f(\xkk) + \frac{M(\xk, \xkk)\etak^2}{2}\norm{\grad(\xk)}_2^2} \\
         & \qquad + 2 \etak \sbr{f(\xopt) - f(\xk) - \frac{\mu_k}{2}\Delta_k}                 \\
         & = \etak^2\rbr{M(\xk, \xkk) \etak - 1} \norm{\grad(\xk)}_2^2
        + 2 \etak \sbr{f(\xopt) - f(\xkk)} - \mu_k \etak \Delta_k                             \\
         & \leq \etak^2\rbr{M(\xk, \xkk) \etak - 1} \norm{\grad(\xk)}_2^2
        - \etak \mu_{k+1} \Delta_{k+1} - \mu_k \etak \Delta_k,                                \\
    \end{align*}
    where the last inequality follows from \( \mu_{k+1} \) strong convexity
    between \( \xkk \) and \( \xopt \).
    Re-arranging this expression allows us to deduce a rate with
    error terms depending on the local smoothness,
    \begin{align*}
        \implies (1 + \mu_{k+1}\etak)\Delta_{k+1}
         & \leq \rbr{1 - \mu_k \etak} \Delta_k
        + \etak^2\rbr{M(\xk, \xkk) \eta - 1} \norm{\grad(\xk)}_2^2                              \\
         & \leq \abs{1 - \mu_k \etak} \Delta_k
        + \etak^2\rbr{M(\xk, \xkk) \eta - 1} \norm{\grad(\xk)}_2^2                              \\
        \implies \Delta_{k+1}
         & \leq \frac{\abs{1 - \mu_k \etak}}{1 + \mu_{k+1}\etak} \Delta_k
        + \frac{\etak^2\rbr{M(\xk, \xkk) \eta - 1}}{1 + \mu_{k+1}\etak} \norm{\grad(\xk)}_2^2   \\
         & \leq \sbr{\prod_{i=0}^k \frac{\abs{1 - \mu_i \eta_i}}{1 + \mu_{i+1}\eta_i}} \Delta_0 \\
         & \hspace{2em} + \sum_{i=0}^k \sbr{\prod_{j = i+1}^k
            \frac{\abs{1 - \mu_j \eta_j}}{1 + \mu_{j+1}\eta_j}}
        \frac{\eta_i^2 \rbr{M(x_i, x_{i+1}) \eta_i - 1}}{1 + \mu_{i+1}\eta_i} \norm{\grad(\x_i)}_2^2.
    \end{align*}
\end{proof}

\convexAnyStepsize*
\begin{proof}
    Let \( \Delta_k = \norm{\xk - \xopt}_2^2 \) and observe \[ \Delta_k =
        \norm{\xk - \xkk + \xkk - \xopt}_2^2 = \Delta_{k+1} + \norm{\xk - \xkk}_2^2 +
        2\abr{\xk - \xkk, \xkk - \xopt}.
    \]
    Using this expansion in \( \Delta_{k+1} - \Delta_k \), we obtain
    \begin{align*}
        \Delta_{k+1} - \Delta_k
         & = -\norm{\xk - \xkk}_2^2
        -2 \abr{\xk - \xkk, \xkk - \xopt}                                                             \\
         & = - \etak^2 \norm{\grad(\xk)}_2^2
        - 2 \etak \abr{\grad(\xk), \xkk - \xopt}                                                      \\
         & = - \etak^2 \norm{\grad(\xk)}_2^2
        - 2 \etak \abr{\grad(\xk), \xkk - \xk}
        - 2 \etak \abr{\grad(\xk), \xk - \xopt}.
        \intertext{
            Now we use convexity and directional smoothness to control
            the two inner-products as follows:
        }
        \Delta_{k+1} - \Delta_k
         & \leq - \etak^2 \norm{\grad(\xk)}_2^2
        - 2 \etak \rbr{f(\xk) - f(\xopt)}
        - 2 \etak \abr{\grad(\xk), \xkk - \xk}                                                        \\
         & \leq - \etak^2 \norm{\grad(\xk)}_2^2
        - 2 \etak \rbr{f(\xk) - f(\xopt)}
        + 2 \etak (f(\xk) - f(\xkk))                                                                  \\
         & \hspace{3em} + \etak^3 M(\xk, \xkk) \norm{\grad(\xk)}_2^2                                  \\
         & = \etak^2 (\etak M(\xk, \xkk) - 1)\norm{\grad(\xk)}_2^2 - 2\etak \rbr{f(\xkk) - f(\xopt)}.
    \end{align*}
    Re-arranging this equation and summing over iterations implies the following
    sub-optimality bound:
    \begin{align*}
        \sum_{i=0}^k \frac{\eta_{i}}{\sum_{i=0}^k \eta_i} \rbr{f(x_{i+1}) - f(\xopt)}
         & \leq \frac{\Delta_{0} + \sum_{i=0}^{k} \eta_i^2 (\eta_i M(x_i, x_{i+1}) - 1)\norm{\grad(\x_i)}_2^2}{2\sum_{i=0}^k \eta_i}.
        \intertext{Convexity of \( f \) and Jensen's inequality now imply the final result,}
        \implies
        f(\bar x_k) - f(\xopt)
         & \leq \frac{\Delta_{0} + \sum_{i=0}^{k} \eta_i^2 (\eta_i M(x_i, x_{i+1}) - 1)\norm{\grad(\x_i)}_2^2}{2\sum_{i=0}^k \eta_i}.
    \end{align*}
\end{proof}

\subsection{Path-Dependent Acceleration: Proofs}\label{app:acceleration}

This section proves \cref{thm:adapted-acceleration} using estimating sequences.
Throughout this section, we assume \( \mu \geq 0 \) is the global strong convexity
parameter, where \( \mu = 0 \) covers the non-strongly convex case.
We start from the estimating sequences version of \cref{eq:agd}, which is
given as follows:
\begin{equation}\label{eq:estimating-seq-acceleration}
    \begin{aligned}
        \alpha_k^2
         & = \etak (1 - \alpha_k) \gamma_k + \etak \alpha_k \mu                               \\
        \gamma_{k+1}
         & = (1 - \alpha_k) \gamma_k + \alpha_k \mu                                           \\
        \yk
         & = \frac{1}{\gamma_k + \alpha_k \mu} \sbr{\alpha_k \gamma_k \vk + \gamma_{k+1} \xk} \\
        \xkk
         & = \yk - \etak \grad(\yk)                                                           \\
        \vkk
         & = \frac{1}{\gamma_{k+1}}
        \sbr{(1 - \alpha_k) \gamma_k \vk + \alpha_k \mu \yk - \alpha_k \grad(\yk)}.
    \end{aligned}
\end{equation}
The algorithm is initialized with \( \x_0 = v_0 \) and some \( \gamma_0 > 0 \).
Note that \( y_0 = x_0 = v_0 \) since \( \yk \) is a convex combination
of \( \xk \) and \( \vk \).
First we prove that this scheme is equivalent to the one given in \cref{eq:agd}.

\begin{lemma}\label{lemma:agd-equivalence}
    \cref{eq:estimating-seq-acceleration} and \cref{eq:agd} lead to equivalent
    updates for the \( \yk \), \( \xk \), and \( \alpha_k \) sequences.
    Moreover, given initialization \( \gamma_0 > 0 \), the corresponding
    initialization for \( \alpha_0 \) is,
    \begin{equation}\label{eq:eta-init}
        \alpha_0 = \frac{\eta_0}{2} \rbr{(\mu - \gamma_0)
            + \sqrt{(\gamma_0 - \mu)^2 + 4 \gamma_0 / \eta_0}}.
    \end{equation}
\end{lemma}
\begin{proof}
    The proof follows \citet[Theorem 2.2.3]{nesterov2018lectures}.
    Expanding the definition of \( \vkk \), we obtain
    \begin{align*}
        \vkk
         & = \frac{1}{\gamma_{k+1}}\sbr{\frac{(1 - \alpha_k)}{\alpha_k}
        \sbr{(\gamma_k + \alpha_k \mu) y_k - \gamma_{k+1} x_k} + \alpha_k \mu y_k - \alpha_k \grad(\yk)} \\
         & = \frac{1}{\gamma_{k+1}}\sbr{\frac{(1 - \alpha_k)\gamma_k}{\alpha_k} \yk + \mu \yk}
        - \frac{(1 - \alpha_k)}{\alpha_k} \xk - \frac{\alpha_k}{\gamma_{k+1}} \grad(\yk)                 \\
         & = \xk - \frac{\etak}{\alpha_k} \grad(\yk) + \frac{1}{\alpha_k}(\yk - \xk)                     \\
         & = \xk + \frac{1}{\alpha_k}\rbr{\xkk - \xk}.
    \end{align*}
    Plugging this back into the expression for \( \ykk \),
    \begin{align*}
        \ykk
         & = \frac{1}{\gamma_{k+1} + \alpha_{k+1} \mu}
        \sbr{\alpha_{k+1} \gamma_{k+1} \vkk + \gamma_{k+2} \xkk}                       \\
         & = \frac{1}{\gamma_{k+1} + \alpha_{k+1} \mu} \sbr{\alpha_{k+1} \gamma_{k+1}
        (\xk + \frac{1}{\alpha_k}\rbr{\xkk - \xk}) + \gamma_{k+2} \xkk}                \\
         & =  \frac{\alpha_{k+1}\gamma_{k_1} + \alpha_k(1 - \alpha_{k+1}) \gamma_{k+1}
            + \alpha_k \alpha_{k+1} \mu}{\alpha_k(\gamma_{k+1} + \alpha_{k+1} \mu)} \xkk
        - \frac{\alpha_{k+1} \gamma_{k+1}(1 - \alpha_k)}
        {\alpha_k(\gamma_{k+1} + \alpha_{k+1} \mu)} \xk                                \\
         & = \xkk +
        \frac{\alpha_{k+1} \gamma_{k+1}(1 - \alpha_k)}
        {\alpha_k(\gamma_{k+1} + \alpha_{k+1} \mu)} (\xkk - \xk)                       \\
         & = \xkk + \frac{\alpha_{k+1} \gamma_{k+1}(1 - \alpha_k)}
        {\alpha_k \rbr{\gamma_{k+1} + \alpha_{k+1}^2 / \etak - (1 - \alpha_{k+1})\gamma_{k+1}}}
        (\xkk - \xk)                                                                   \\
         & = \xkk + \frac{\alpha_k(1 - \alpha_k)}
        {\rbr{\alpha_{k+1} + \alpha_k^2}}(\xkk - \xk).
    \end{align*}
    Note that this update is consistent with \cref{eq:agd}.
    Since \( \gamma_k = \alpha_k^2 / \etak \), we can write,
    \begin{align*}
        \alpha_{k+1}^{2}
         & = \etakk (1 - \alpha_{k+1}) \gamma_{k} + \etakk \alpha_{k+1} \mu                \\
         & = \frac{\etakk}{\etak} (1 - \alpha_{k+1}) \alpha_k^2 + \etakk \alpha_{k+1} \mu,
    \end{align*}
    which is also consistent with \cref{eq:agd}.
    Finally, the initialization for \( \alpha_0 \) is determined by
    \( \gamma_0 \) in \cref{eq:estimating-seq-acceleration} as,
    \[
        \alpha_0^2 = \eta_0 (1 - \alpha_0) \gamma_0 + \eta_0 \alpha_0 \mu.
    \]
    The quadratic formula now implies,
    \[
        \alpha_0 = \frac{\eta_0}{2} \rbr{(\mu - \gamma_0) + \sqrt{(\gamma_0 - \mu)^2 + 4 \gamma_0 / \eta_0}}.
    \]
    This completes the proof.
\end{proof}

As mentioned, our proof uses the concept of estimating sequences.

\begin{definition}
    Two sequences \( \lambda_k \), \( \phi_k \) are estimating sequences for \( f \)
    if \( \lambda_l \geq 0 \) for all \( k \in \bbN \), \( \lim_{k\rightarrow \infty} \lambda_k = 0 \),
    and,
    \begin{equation}\label{eq:est-upper-bound}
        \phi_k(x) \leq (1 - \lambda_k) f(x) + \lambda_k \phi_0(x),
    \end{equation}
    for all \( x \in \R^d \).
\end{definition}

We use the same estimating sequences as developed by \citet{nesterov2018lectures}.
Let \( \lambda_0 = 1 \), \( \phi_0(x) = f(x_0) + \frac{\gamma_0}{2} \norm{x - x_0}_2^2 \),
and define the updates,
\begin{equation}\label{eq:estimating-sequences}
    \begin{aligned}
        \lambda_{k+1}
         & = (1 - \alpha_k) \lambda_k \\
        \phi_{k+1}(x)
         & = (1 - \alpha_k) \phi_k(x)
        + \alpha_k \rbr{f(\yk) + \abr{\grad(\yk), x - \yk} + \frac{\mu}{2}\norm{x - \yk}_2^2},
    \end{aligned}
\end{equation}
where \( \mu \geq 0 \) is the strong convexity parameter, with \( \mu = 0 \)
when \( f \) is merely convex.
It is straightforward to differentiate \( \phi_{k+1} \) to see that
\( v_{k+1} \) of \cref{eq:estimating-seq-acceleration} is the minimizer.
Indeed, \citet[Lemma 2.2.3]{nesterov2018lectures} shows that this choice for the
estimating sequences obeys the following canonical form:
\begin{equation}
    \phi_{k+1}(x) = \min_{z} \phi_{k+1}(z) + \frac{\gamma_{k+1}}{2}\norm{x - \vkk}_2^2,
\end{equation}
where \( \gamma_{k+1} \) and \( \vkk \) are given by~\cref{eq:estimating-seq-acceleration}
and the minimum value is,
\begin{equation}\label{eq:canonical-minimizer}
    \begin{aligned}
        \min_{z} \phi_{k+1}(z)
         & = (1 - \alpha_k) \min_{z} \phi_k(z)
        + \alpha_k f(\yk) - \frac{\alpha_k^2}{2 \gamma_{k+1}} \norm{\grad(\yk)}_2^2 \\
         & \hspace{2em} + \frac{\alpha_k (1 - \alpha_k) \gamma_k}{\gamma_{k+1}}
        \rbr{\frac{\mu}{2} \norm{\yk - \vk}_2^2 + \abr{\grad(\yk), \vk - \yk}}.
    \end{aligned}
\end{equation}

Before we can prove our main theorem, we must show that these choices
for \( \lambda_k \) and \( \phi_k \) yield a valid estimating sequence.
The following proofs build on \citep{nesterov2018lectures}
and \citep{mishkin2024faster}.

\begin{lemma}\label{lemma:lambda-convergence}
    Assume \( \alpha_k \in (0, 1] \) for all \( k \in \bbN \).
    If \( \mu > 0 \) and \( \gamma_0 = \mu \), then
    \begin{equation}
        \label{eq:lambda-sc}
        \lambda_k = \prod_{i=0}^{k-1}(1 - \sqrt{\etak \mu}).
    \end{equation}
    If \( \mu \geq 0 \) and \( \gamma_0 \in (\mu, \mu + 3 / \etamin) \), then,
    \begin{equation}
        \label{eq:lambda-convex}
        \lambda_k \leq \frac{4}{\etamin (\gamma_0 - \mu)(k+1)^2}.
    \end{equation}
\end{lemma}
\begin{proof}
    Assume \( \gamma_0 = \mu > 0 \).
    Then \( \gamma_k = \mu \) for all \( k \) and,
    \begin{align*}
        \alpha_{k}^2
         & = (1 - \alpha_{k}) \etak \mu + \alpha_k \etak \mu \\
         & = \etak \mu.
    \end{align*}
    As a consequence,
    \[
        \lambda_k = \prod_{i=0}^{k-1} (1 - \sqrt{\etak \mu}),
    \]
    as claimed.

    Now assume \( \gamma_0 \in \rbr{\mu, 3 L + \mu} \).
    Modifying the proof by \citet[Lemma 2.2.4]{nesterov2018lectures}, we
    compute as follows:
    \begin{align*}
        \gamma_{k+1} - \mu
         & = (1 - \alpha_k) \gamma_k + (\alpha_k - 1) \mu
        = (1 - \alpha_k) (\gamma_k - \mu)                  \\
        \intertext{
            Recursing on this equality implies
        }
        \gamma_{k+1}
         & = (\gamma_0 - \mu) \prod_{i=0}^k (1 - \alpha_k)
        = \lambda_{k+1} (\gamma_0 - \mu).
    \end{align*}
    If \( \alpha_k = 1 \) or \( \lambda_k = 0 \), then
    using \( \lambda_{k+1} = (1 - \alpha_k) \lambda_k \)
    implies \( \lambda_{k+1} = 0 \) and the result trivially holds.
    Otherwise, recall \( \alpha_k^2 / \gamma_{k+1} = \etak \) to obtain,
    \begin{align*}
        1 - \frac{\lambda_{k+1}}{\lambda_k}
         & = \alpha_k
        = \rbr{\gamma_{k+1} \etak}^{1/2}                                     \\
         & = \rbr{\etak \mu + \etak \lambda_{k+1} (\gamma_0 - \mu)}^{1/2}    \\
        \implies
        \frac{1}{\lambda_{k+1}} - \frac{1}{\lambda_k}
         & = \frac{1}{\lambda_{k+1}^{1/2}}
        \sbr{\frac{\etak \mu}{\lambda_{k+1}} + \etak (\gamma_0 - \mu)}^{1/2} \\
         & \geq \frac{1}{\lambda_{k+1}^{1/2}}
        \sbr{\frac{\etamin \mu}{\lambda_{k+1}} + \etamin (\gamma_0 - \mu)}^{1/2}.
    \end{align*}
    Finally, this implies
    \begin{align*}
        \frac{2}{\lambda_{k+1}^{1/2}}
        \rbr{\frac{1}{\lambda_{k+1}^{1/2}} - \frac{1}{\lambda_k^{1/2}}}
         & \geq \rbr{\frac{1}{\lambda_{k+1}^{1/2}} - \frac{1}{\lambda_k^{1/2}}}
        \rbr{\frac{1}{\lambda_{k+1}^{1/2}} + \frac{1}{\lambda_k^{1/2}}}         \\
         & \geq
        \frac{1}{\lambda_{k+1}^{1/2}}\sbr{\frac{\etamin \mu}{\lambda_{k+1}} + \etamin (\gamma_0 - \mu)}^{1/2}.
    \end{align*}
    Moreover, this bound holds uniformly for all \( k \in \bbN \).
    We have now exactly reached Eq. 2.2.11 of
    \citet[Lemma 2.2.4]{nesterov2018lectures} with \( L \) replaced by \( 1 / \etamin \).
    Applying that Lemma with this modification, we obtain
    \[
        \lambda_k \leq
        \frac{4}{\etamin (\gamma_0 - \mu)(k+1)^2},
    \]
    which completes the proof.
\end{proof}

\begin{lemma}\label{lemma:valid-sequences}
    If \( f \) is strongly convex with parameter \( \mu \geq 0 \) and
    \( \etak \leq 1 / \mu \) for all \( k \in \bbN \),
    then \( \lambda_k \) and \( \phi_k \) are estimating sequences.
\end{lemma}
\begin{proof}
    Using the quadratic formula, we find
    \[
        \alpha_k
        = \frac{\mu - \gamma_k \pm
            \sqrt{(\mu - \gamma_k)^2 + 4 \hat \gamma_k / \etak}}{2 / \etak}.
    \]
    Thus,
    \begin{align*}
        (\mu - \gamma_k) + \rbr{(\mu - \gamma_k)^2 + 4 \hat \gamma_k / \etak}^{1/2}
         & > 0.
        \intertext{is sufficient for \( \alpha_k > 0 \).
            This holds if \( \mu \geq \gamma_k \).
            Otherwise, we require,
        }
        (\mu - \gamma_k)^2 + 4 \hat \gamma_k / \etak
         & > (\mu - \gamma_k)^2,
    \end{align*}
    which holds if and only if \( \etak, \gamma_k > 0 \).
    On the other hand, we also need \( \alpha_k \leq 1 \),
    which is satisfied when,
    \begin{align*}
        4 + 4\etak (\gamma_k - \mu) + \etak^2 (\mu - \gamma_k)^2
         & \leq \etak^2 (\mu - \gamma_k)^2 + 4 \etak \gamma_k
        \iff
        \etak \leq \frac{1}{\mu},
    \end{align*}
    as claimed.

    Recall \( \lambda_0 = 1 \) and \( \lambda_{k+1} = (1 - \alpha_k) \lambda_k \).
    Since \( \alpha_{k} \in (0, 1] \), \( \lambda_k \geq 0 \) holds by induction.
    It remains to show that \( \lambda_k \) tends to zero, which holds
    by \cref{lemma:lambda-convergence} since we have shown
    \( \alpha_k \in (0, 1] \) for all \( k \).

    Now we establish the last piece,
    \[
        \phi_k(x) \leq (1 - \lambda_k) f(x) + \lambda_k \phi_0(x).
    \]
    But this follows immediately by \citet[Lemma 2.2.2]{nesterov2018lectures}.
\end{proof}

Now we can prove the last major lemma before our convergence result.

\begin{lemma}\label{lemma:local-upper-bound}
    Suppose \( f \) is strongly convex with
    parameter \( \mu \geq 0 \) and \( \etak \) is a sequence
    of adapted step-sizes, meaning \( \etak \leq 1 / M(\xk, \xkk) \).
    Then for every \( k \in \bbN \),
    \begin{equation*}
        \min_z \phi_k(z) \geq f(\xk).
    \end{equation*}
\end{lemma}
\begin{proof}
    We use an inductive proof again.
    The inductive assumption is
    \[
        \min_{z} \phi_k(z) \geq f(\xk),
    \]
    It is easy to see this holds at \( k = 0 \) since,
    \[
        \phi_0(x) = f(x_0) + \frac{\gamma_0}{2}\norm{x - v_0}_2^2,
    \]
    implies \( \min_z \phi_0(z) = f(x_0) \).
    Using \cref{eq:canonical-minimizer}, we obtain
    \begin{align*}
        \min_z \phi_{k+1}(z)
         & = (1- \alpha_k) \min_z \phi_k(z) + \alpha_k f(\yk)
        - \frac{\alpha_k^2}{2\gamma_{k+1}}\norm{\grad(\yk)}^2                      \\
         & \hspace{5em} + \frac{\alpha_k(1-\alpha_k)\gamma_k}{\gamma_{k+1}}
        \rbr{\frac{\mu}{2} \norm{\yk - \vk}^2 + \abr{\grad(\yk, \zk), \vk - \yk}}  \\
         & \geq (1- \alpha_k) f(\xk) + \alpha_k f(\yk)
        - \frac{\alpha_k^2}{2\gamma_{k+1}}\norm{\grad(\yk)}^2                      \\
         & \hspace{5em} + \frac{\alpha_k(1-\alpha_k)\gamma_k}{\gamma_{k+1}}
        \rbr{\frac{\mu}{2} \norm{\yk - \vk}^2 + \abr{\grad(\yk), \vk - \yk}},
        \intertext{
            where the inequality holds by the inductive assumption.
            Using convexity of \( f \) and
            recalling \( \frac{\alpha_k^2}{\gamma_{k+1}} = \etak \) from
            the definition of the update (\cref{eq:estimating-seq-acceleration}), }
        \min_{z} \phi_{k+1}(z)
         & \geq (1- \alpha_k)\rbr{f(\yk) + \abr{\grad(\yk), \xk - \yk}} + \alpha_k
        f(\yk) - \frac{\etak}{2}\norm{\grad(\yk)}^2                                \\
         & \hspace{5em} +
        \frac{\alpha_k(1-\alpha_k)\gamma_k}{\gamma_{k+1}} \rbr{\frac{\mu}{2}
        \norm{\yk - \vk}^2 + \abr{\grad(\yk), \vk - \yk}}                          \\
         & = f(\yk) +
        (1-\alpha_k) \abr{\grad(\yk), \xk - \yk} -
        \frac{\etak}{2}\norm{\grad(\yk)}^2                                         \\
         & \hspace{5em} +
        \frac{\alpha_k(1-\alpha_k)\gamma_k}{\gamma_{k+1}} \rbr{\frac{\mu}{2}
        \norm{\yk - \vk}^2 + \abr{\grad(\yk), \vk - \yk}}                          \\
        \intertext{Using the fact that the step-sizes are adapted
            and invoking the directional descent lemma (i.e. \cref{eq:agd-descent-condition})
            now implies }
        \min_{z} \phi_{k+1}(z)
         & \geq
        f(\xkk) + (1-\alpha_k) \bigg(\abr{\grad(\yk), \xk - \yk}                   \\
         &
        \hspace{5em} + \frac{\alpha_k\gamma_k}{\gamma_{k+1}}\rbr{\frac{\mu}{2}
            \norm{\yk - \vk}^2 + \abr{\grad(\yk), \vk - \yk}}\bigg).
        \\
        \intertext{
            The remainder of the proof is largely unchanged from the analysis
            in \citet{nesterov2018lectures}.
            The definition of \( \yk \) gives \( \xk - \yk = \frac{\alpha_k
                \gamma_k}{\gamma_{k+1}} (\yk - \vk) \), which we use to obtain }
        \min_{z} \phi_{k+1}(z)
         & \geq f(\xkk) + (1-\alpha_k)
        \bigg(\frac{\alpha_k \gamma_k}{\gamma_{k+1}} \abr{\grad(\yk), \yk -
        \vk}                                                                       \\
         & \hspace{5em} +
        \frac{\alpha_k\gamma_k}{\gamma_{k+1}}\Big(\frac{\mu}{2} \norm{\yk - \vk}^2 +
        \abr{\grad(\yk), \vk - \yk}\Big)\bigg)                                     \\
         & = f(\xkk) + \frac{\mu
        \alpha_k(1-\alpha_k)\gamma_k}{2 \gamma_{k+1}}\norm{\yk - \vk}^2            \\
         & \geq
        f(\xkk),
    \end{align*}
    since \( \frac{\mu \alpha_k(1-\alpha_k)\gamma_k}{2 \gamma_{k+1}} \geq 0 \).
    We conclude the desired result by induction.
\end{proof}

The main accelerated result now follows almost immediately.

\adaptedAcceleration*
\begin{proof}
    We analyze the equivalent formulation given in \cref{eq:estimating-seq-acceleration}.
    See \cref{lemma:agd-equivalence} for a formal proof that these two
    schemes produce the same \( \xk \), \( \yk \), and \( \alpha_k \) iterates.
    Note that our proof follows \citet{nesterov2018lectures} and \citet{mishkin2024faster}
    closely; while their results are very similar, we are not aware of
    pre-existing works which adapt them to our specific setting.

    First, observe that \( M(\xk, \xkk) \geq \mu \) for all \( k \in \bbN \).
    Since the step-sizes \( \etak \) are assumed to satisfy
    \( \etak \leq 1 / M(\xk, \xkk) \), we also have that
    \( \etak \leq 1 / \mu \) for every \( k \).

    Thus, \cref{lemma:lambda-convergence} and \cref{lemma:valid-sequences}
    apply.
    Using the definition of an estimating sequence and \cref{lemma:local-upper-bound},
    we obtain,
    \begin{align*}
        f(\xk)
         & \leq \min_{z} \phi_k(z)                                                          \\
         & \leq \min_{z} (1 - \lambda_k) f(z) + \lambda_k \phi_0(z)                         \\
         & \leq (1 - \lambda_k) f(\xopt) + \lambda_k \phi_0(\xopt).
        \intertext{
            Re-arranging this equation and expanding the definition \( \phi_0 \) (\cref{eq:estimating-sequences}), we deduce the following:}
        f(\xk) - f(\xopt)
         & \leq \lambda_k \rbr{\phi_0(\xopt) - f(\xopt)}                                    \\
         & = \lambda_k \rbr{f(x_0) - f(\xopt) + \frac{\gamma_0}{2} \norm{x_0 - \xopt}_2^2}.
    \end{align*}
    We see that the rate of convergence of AGD is entirely controlled by the
    convergence of the sequence \( \lambda_k \).
    If \( \mu > 0 \) and \( \gamma_0 = \mu \), then \cref{lemma:lambda-convergence}
    implies
    \[
        f(\xk) - f(\xopt) \leq \prod_{i=0}^{k-1} \rbr{1 - \sqrt{\mu \eta_i}}
        \sbr{f(\x_0) - f(\xopt) + \frac{\mu}{2}\norm{\x_0 - \xopt}_2^2}.
    \]
    By \cref{lemma:agd-equivalence}, this initialization is equivalent to
    choosing \( \alpha_0 = \sqrt{\eta_0 \mu} \), which is the setting
    claimed in the theorem.

    Alternatively, if \( \mu \geq 0 \) and
    \( \gamma_0 \in (\mu, \mu + 3 / \etamin) \),
    then,
    \[
        f(\xk) - f(\xopt) \leq \frac{4}{\etamin (\gamma_0 - \mu) k^2}
        \sbr{f(\x_0) - f(\xopt) + \frac{\gamma_0}{2}\norm{\x_0 - \xopt}_2^2},
    \]
    where the equality,
    \[
        \gamma_0 = \frac{\alpha_0^2 - \eta_0 \alpha_0 \mu}{\eta_0 (1 - \alpha_0)},
    \]
    holds by \cref{lemma:agd-equivalence}.
    Since \( \alpha_0 \leq 1 \) for \( \eta_0 \leq 1 / \mu \), 
    \( \eta_0 \) is an increasing function of \( \gamma_0 \).
    Thus, an upper-bound \( c \) on \( \alpha_0 \) can be deduced from that on \( \gamma_0 \)
    using the quadratic formula:
    \begin{align*}
        c =
         & - \frac{3\eta_0}{2 \etamin}
        + \frac{\eta_0}{2}\rbr{\frac{9}{(\etamin)^2} + 4\frac{3 \etamin + \mu}{\eta_0}}^{1/2} \\
         & = \frac{3\eta_0}{2 \etamin}
        \sbr{\rbr{1 + 4(\etamin)^2 \frac{3 \etamin + \mu}{9 \eta_0}}^{1/2} - 1}.
    \end{align*}
\end{proof}

%% file: appendices/quadratic_case_proofs.tex

\begin{lemma}\label{lemma:quadratic-pointwise-direction-smoothness}
    Let \( B \) be a positive semi-definite matrix and suppose
    that
    \[
        f(x) = \frac{1}{2} x^{\top} B x - c^{\top} x.
    \]
    Let $x_{i+1} = x_i - \eta \nabla f(x_i)$. Then for any $\eta > 0$, the pointwise directional smoothness between the gradient descent iterates $x_i, x_{i+1}$  is given by
    \[
        \frac{1}{2} D(x_i, x_{i+1})
        = \frac{\norm{B \nabla f(x_i)}_2}{\norm{\nabla f(x_i)}_2}.
    \]
\end{lemma}
\begin{proof}
    We have by straightforward algebra,
    \begin{align*}
        \frac{1}{2} D(x_i, x_{i+1})
         & = \frac{\norm{\nabla f(x_{i+1}) - \nabla f(x_i)}_2}{\norm{x_{i+1} - x_i}_2}                       \\
         & = \frac{\norm{\left[ B x_{i+1} - c \right] - \left[ B x_i - c \right]}_2}{\norm{x_{i+1} - x_i}_2} \\
         & = \frac{\norm{B \left[ x_{i+1} - x_i \right]}_2}{\norm{x_{i+1} - x_i}_2}                          \\
         & = \frac{\norm{B \left[ - \eta \nabla f(x_i) \right] }_2}{\norm{-\eta \nabla f(x_i)}_2}            \\
         & = \frac{\norm{B \nabla f(x_i)}_2}{\norm{\nabla f(x_i)}_2}.
    \end{align*}
\end{proof}

\begin{lemma}\label{lemma:quadratic-pathwise-direction-smoothness}
    Let \( B \) be a positive semi-definite matrix and suppose
    that
    \[
        f(x) = \frac{1}{2} x^{\top} B x - c^{\top} x.
    \]
    Let $x_{i+1} = x_i - \eta \nabla f(x_i)$. Then for any $\eta > 0$, the path-wise directional smoothness between the gradient descent iterates $x_i, x_{i+1}$  is given by
    by
    \[
        A(x_i, x_{i+1})
        = \frac{{\nabla f(x_i)^{\top} B \nabla f(x_i)}}{\nabla f(x_i)^{\top} \nabla f(x_i)}.
    \]
\end{lemma}
\begin{proof}
    Let $A_t (x, y) = \frac{\ev{ \nabla f(x+t(y-x)) - \nabla f(x) , y-x }}{t \sqn{x-y}_2}$. We
    have
    \begin{align*}
        A_t(x, y)
         & = \frac{\ev{ \nabla f(x+t(y-x)) - \nabla f(x) , y-x }}{t \sqn{x-y}_2}       \\
         & = \frac{\ev{ (B(x+t(y-x)))-c - \left[ Bx -c \right] , y-x }}{t \sqn{x-y}_2} \\
         & = \frac{\ev{ t \cdot B(y-x) , y-x }}{t \sqn{x-y}_2}                         \\
         & = \frac{(y-x)^{\top} B (y-x)}{\sqn{x-y}_2}.
    \end{align*}
    The path-wise directional smoothness $A$ is therefore
    \begin{align*}
        A (x, y)
         & = \sup_{t \in [0, 1]} A_t (x, y)                             \\
         & = \sup_{t \in [0, 1]} \frac{(y-x)^{\top} B (y-x)}{\sqn{x-y}_2} \\
         & = \frac{(y-x)^{\top} B (y-x)}{\sqn{x-y}_2}.
    \end{align*}
    Plugging in $y = x - \eta \nabla f(x) = x - \eta [Bx-c]$ in the above gives
    \begin{align*}
        A (x, x-\eta\nabla f(x))
         & = \frac{\left( -\eta \left[ Bx-c \right] \right) B (- \eta) \left[ Bx-c \right]}{\sqn{\eta [Bx-c]}_2} \\
         & = \frac{(Bx-c)^{\top} B (Bx-c)}{\sqn{Bx-c}_2}                                                         \\
         & = \frac{(Bx-c)^{\top} B (Bx-c)}{\sqn{Bx-c}_2}                                                         \\
         & = \frac{\nabla f(x)^{\top} B \nabla f(x)}{\nabla f(x)^{\top} \nabla f(x)}.
    \end{align*}
\end{proof}

%% file: appendices/general_case_proofs.tex
\propExistenceD*
\begin{proof}
    Let \( \calI = \cbr{ \eta : \grad(x - \eta \grad(x)) = \grad(x) }\).
    For every \( \eta \in \calI \), it holds that
    \[
        -\abr{\grad(x - \eta \grad(x)), \grad(x)} = -\norm{\grad(x)}_2^2.
    \]
    However, since \( f \) is convex, the directional derivative
    \[
        -\abr{\grad(x - \eta' \grad(x)), \grad(x)},
    \]
    is monotone non-decreasing in \( \eta' \).
    We deduce that \( \calI \) must be an interval of form \( [0, \bar \eta] \).
    If \( \bar \eta \) is not bounded, then \( f \) is linear along
    \( -\grad(x) \) and is minimized by taking \( \eta \rightarrow \infty \).
    Therefore, we may assume \( \bar \eta \) is finite.

    Let \( \eta > \bar \eta \). Then we have the following:
    \begin{align*}
        x - \eta \grad(x)
         & = x
        - \frac{2\norm{x - \eta \grad(x) - x}_2}{\norm{\grad(x - \eta \grad(x)) - \grad(x)}_2}
        \grad(x) \\
        \iff \grad(x)
         & =
        \frac{2\norm{\grad(x)}_2}{\norm{\grad(x - \eta \grad(x)) - \grad(x)}_2}
        \grad(x),
    \end{align*}
    from which we deduce
    \[
        \norm{\grad(x - \eta \grad(x)) - \grad(x)}_2 = 2\norm{\grad(x)}_2,
    \]
    is sufficient for the implicit equation to hold.
    Squaring both sides and multiplying by \( 1/2 \), we obtain the following
    alternative root-finding problem:
    \begin{equation}\label{eq:zero-finding}
        h(\eta) := \half \norm{\grad(x - \eta \grad(x))}_2^2
        - \abr{\grad(x - \eta \grad(x)), \grad(x)} - \half \norm{\grad(x)}_2^2 = 0.
    \end{equation}
    Because \( f \) is \( C^1 \), \( h \) is a continuous function and
    it suffices to show that there exists an interval in which \( h \)
    crosses \( 0 \).
    From the display above, we see
    \[
        h(\bar \eta) = - \norm{\grad(x)}_2^2 < 0.
    \]
    Continuity now implies \( \exists \eta' > \bar \eta \) such that
    \( h(\eta') < 0 \).
    Now, suppose \( h(\eta) \leq 0 \) for all \( \eta \geq \eta' \).
    Working backwards, we see that this can only occur when
    \[
        \eta
        \leq
        \frac{2\norm{x - \eta \grad(x) - x}_2}{\norm{\grad(x - \eta \grad(x)) - \grad(x)}_2}
        = \frac{1}{D(x(\eta), x - \eta \grad(x))}
    \]
    for all \( \eta \geq \eta' \).
    The directional descent lemma (\cref{eq:local-descent-lemma}) now
    implies
    \[
        f(x - \eta \grad(x)) \leq
        f(x) - \eta\rbr{1 - \frac{\eta D(x, x - \eta \grad(x))}{2}}
        \norm{\grad(x)}_2^2
        \leq f(x) - \frac{\eta}{2}\norm{\grad(\x)}_2^2,
    \]
    Taking limits on both sides as \( \eta \rightarrow \infty \)
    implies \( f(x - \eta \grad(\x)) \) is minimized along the
    ray \( x(\eta) = x -  \eta \grad(x) \).
    Thus, we deduce that either there exists \( \eta'' > \eta' \) such that \(
    h(\eta'') > 0 \) exists, or \( f \) is minimized along the gradient
    direction as claimed.
\end{proof}

\propExistenceA*
\begin{proof}
    Let
    \[
        \calJ = \cbr{ \eta : \abr{\grad(x - \eta \grad(x)), \grad(x)}
            = \norm{\grad(x)}_2^2 }.
    \]
    Since \( f \) is convex, the directional derivative
    \[
        -\abr{\grad(x - \eta' \grad(x)), \grad(x)},
    \]
    is monotone non-decreasing in \( \eta' \).
    We deduce that \( \calJ \) must be an interval of form \( [0, \bar \eta] \).
    If \( \bar \eta \) is not bounded, then convexity implies
    \begin{align*}
        \lim_{\eta \rightarrow \infty} f(x - \eta \grad(x))
         & \leq
        \lim_{\eta \rightarrow \infty}
        f(x) - \eta \abr{\grad(x - \eta \grad(x)), \grad(x)} \\
         & = -\infty,
    \end{align*}
    meaning \( f \)	 is minimized along \( - \grad(x) \).
    Therefore, we may assume \( \bar \eta \) is finite.

    We have
    \begin{align*}
        x - \eta \grad(x)
         & = x - \frac{1}{A(x, x-\eta \grad(x))} \grad(x) \\
        \iff \eta
         & = \inf_{t \in [0, 1]}
        \frac{t \eta \norm{\grad(x)}_2^2}
        {\abr{\grad(x) - \grad(x - t \eta \grad(x)), \grad(x)}}.
    \end{align*}
    Thus, for \( \eta > \bar \eta \), the equation we must solve reduces to
    \[
        h(\eta) := \eta - \inf_{t \in [0, 1]}
        \frac{t \eta \norm{\grad(x)}_2^2}
        {\abr{\grad(x) - \grad(x - t \eta \grad(x)), \grad(x)}}
        = 0.
    \]
    Since \( f \) is \( C^2 \), \( h \) is continuous
    (see, e.g. \citet[Theorem 7]{hogan1973point})
    and it suffices to show that there exists an interval over which \( h \)
    crosses \( 0 \).

    Using Taylor's theorem, we can re-write this expression as
    \begin{align*}
        h(\eta) & = \eta - \inf_{t \in [0, 1]}
        \frac{\norm{\grad(x)}_2^2}
        {\abr{\grad(x), \nabla^2 f(x - \alpha(t \eta) \grad(x)) \grad(x)}},
    \end{align*}
    where for some \( \alpha(t \eta) \in [0, t \eta] \).
    Examining the denominator, we find that,
    \[
        \int_{0}^t \grad(x)^\top \nabla^2 f(x - t \bar \eta \grad(x)) \grad(x) dt
        = \abr{\grad(x - \bar \eta \grad(x)) - \grad(x), \grad(x)} = 0,
    \]
    which, since \( f \) is convex, implies
    \[
        \grad(x)^\top \nabla^2 f(x - \alpha \grad(x)) \grad(x) = 0,
    \]
    for every \( \alpha \in [0, \bar \eta] \).
    By continuity of the Hessian,
    for every \( \epsilon > 0 \), there exists \( \delta > 0 \)
    such that
    \( \eta' \in [\bar \eta, \bar \eta + \delta] \)
    guarantees,
    \[
        \grad(x)^\top \nabla^2 f(x - \eta' \grad(x)) \grad(x) < \epsilon.
    \]
    Substituting this into our expression for \( h \),
    \begin{align*}
        h(\eta') & = \eta' - \inf_{t \in [0, 1]}
        \frac{\norm{\grad(x)}_2^2}
        {\abr{\grad(x), \nabla^2 f(x - \alpha(t \eta') \grad(x)) \grad(x)}}    \\
                 & < \bar \eta + \delta - \frac{\norm{\grad(x)}_2^2}{\epsilon} \\
                 & < 0,
    \end{align*}
    for \( \epsilon, \delta \) sufficiently small.
    Thus, there exists \( \eta' > \bar \eta \) for which \( h(\eta') < 0 \).

    Now let us show that \( h(\eta'') > 0 \) for some \( \eta'' \).
    For convenience, define
    \[
        g(\eta) = \inf_{t \in [0, 1]}
        \frac{t \norm{\grad(x)}_2^2}
        {\abr{\grad(x) - \grad(x - t \eta \grad(x)), \grad(x)}},
    \]
    which is a continuous and monotone non-increasing function.
    Take \( \eta \rightarrow \infty \) and let
    \[
        \lim_{\eta \rightarrow \infty} g(\eta) = c,
    \]
    where the limit exists, but may be \( - \infty \).
    Indeed, it must hold that \( c < \infty \) since,
    \begin{align*}
        \lim_{\eta \rightarrow \infty} g(\eta) < g(\eta') < \infty.
    \end{align*}
    If \( c < 0 \), then taking \( \eta'' \)
    large enough that \( g(\eta'') \leq 0 \) suffices.
    Alternatively, if \( c \geq 0 \), then there exists \( \tilde \eta \)
    such that \( g(\eta) \leq c + \epsilon \) for every
    \( \eta \geq \tilde \eta \).
    Choosing \( \eta'' > \max\cbr{\tilde \eta, c} + \epsilon \) yields
    \[
        h(\eta'') = \eta'' - g(\eta'') > c + \epsilon - c - \epsilon = 0.
    \]
    This completes the proof.
\end{proof}

\begin{algorithm}[t]
    \caption{Gradient Descent with Exponential Search}
    \label{alg:exponential-search-gd}
    \begin{algorithmic}[1]
        \STATE \textbf{Procedure} ExponentialSearch($x, \eta_0$)
        \FOR{$k = 1, 2, 3, \ldots$}
        \STATE $\eta_{\mathrm{out}} \gets \text{RootFindingBisection}\rbr{x, 2^{-2^{k}} \eta_0, \eta_0}$.
        \IF {$\eta_{\mathrm{out}} < \infty$}
        \STATE \textbf{Return} $\eta_{\mathrm{out}}$
        \ENDIF
        \ENDFOR
        \STATE \textbf{End Procedure}

        \STATE \textbf{Procedure} RootFindingBisection($x, \eta_{\mathrm{lo}}, \eta_{\mathrm{hi}}$)
        \STATE Define $\phi(\eta) = \eta - \psi(\eta)$ where $\psi(\eta)$ is given in~\eqref{eq:psi}
        {\color{gray} $\backslash\backslash$ One access to $\phi$ requires $T$ descent steps.}

        \IF {$\phi(\eta_{\mathrm{hi}}) \leq 0$}
        \STATE \textbf{Return} $\eta_{\mathrm{hi}}$
        \ENDIF

        \IF {$\phi(\eta_{\mathrm{lo}}) > 0$}
        \STATE \textbf{Return} $\infty$
        \ENDIF

        \WHILE{$\eta_{\mathrm{hi}} > 2 \eta_{\mathrm{lo}}$}
        \STATE $\eta_{\mathrm{mid}} = \sqrt{\eta_{\mathrm{lo}} \eta_{\mathrm{hi}}}$

        \IF {$\phi(\eta_{\mathrm{mid}}) > 0$}
        \STATE $\eta_{\mathrm{hi}} = \eta_{\mathrm{mid}}$
        \ELSE
        \STATE $\eta_{\mathrm{lo}} = \eta_{\mathrm{mid}}$
        \ENDIF

        {\color{gray}  $\backslash\backslash$ Invariant: $\phi(\eta_{\textrm{hi}}) > 0$, and $\phi(\eta_{\mathrm{lo}}) \leq 0$.}
        \ENDWHILE

        \STATE \textbf{Return} $\eta_{\mathrm{lo}}$
        \STATE \textbf{End Procedure}
    \end{algorithmic}
\end{algorithm}

\exponentialsearch*
\begin{proof}[Proof of Theorem~\ref{thm:exponential-search-gd}]
    This analysis follows \citep{yair2022parameter}.
    First, instantiate~\cref{eq:1} from
    \Cref{prop:convex-any-stepsize} with $\eta_i = \eta$ for all $i$ to obtain
    \begin{align}
        \label{eq:exp-search-1}
        f(\bar{x}_k) - f^{*} \leq \frac{\norm{x_0 - x^{*}}^2}{2 \eta k} + \frac{\eta \left[ \eta \sum_{i=0}^k M(x_i, x_{i+1}) \norm{\nabla f(x_i)}^2 - \sum_{i=0}^k \norm{\nabla f(x_i)}^2 \right]}{2 k}.
    \end{align}
    Now, observe that if we get a ``Lucky strike'' and $\phi(\eta_{\mathrm{hi}}) = \phi(\eta_0) \leq 0$, then specializing \cref{eq:exp-search-1} for $\eta = \eta_0$ we get
    \begin{align*}
        f(\bar{x}_k) - f(\xopt) & \leq \frac{\sqn{x_0 - \xopt}_2}{2\eta_0 k} + \frac{\eta_0}{2k} \left[ \eta_0 \sum_{i=0}^k M(x_i, x_{i+1}) \sqn{\nabla f(x_i)}_2 - \sum_{i=0}^k \sqn{\nabla f(x_i)}_2  \right] \\
                                & = \frac{\sqn{x_0 - \xopt}_2}{2\eta_0 k} + \frac{\eta_0 \sum_{i=0}^k M(x_i, x_{i+1}) \sqn{\nabla f(x_i)}_2}{2k} \cdot \phi(\eta_0)                                             \\
                                & \leq \frac{\sqn{x_0 - \xopt}_2}{2 \eta_0 k}.
    \end{align*}
    This covers the first case of Theorem~\ref{thm:exponential-search-gd}.

    With the first case out of the way, we may assume that
    $\phi(\eta_{\mathrm{hi}}) > 0$.
    This implies that $\eta_{\mathrm{hi}} > \frac{1}{L}$, since if $\eta \leq
        \frac{1}{L}$ we have $\phi(\eta) \leq 0$.
    Now observe that when $\eta_{\mathrm{lo}} = 2^{2^{-k}} \eta_0 \leq
        \frac{1}{L}$, we have that $\phi(\eta_{\mathrm{lo}}) \leq 0$, therefore it
    takes at most $k = \ceil{\log \log \frac{\eta_0}{L^{-1}}}$ to find such an
    $\eta_{\mathrm{lo}}$.
    From here on, we suppose that $\phi(\eta_{\mathrm{hi}}) > 0$ and
    $\phi(\eta_{\mathrm{lo}}) \leq 0$.
    Now observe that the algorithm's main loop always maintains the invariant
    $\phi(\eta_{\mathrm{hi}}) > 0$ and $\phi(\eta_{\mathrm{lo}}) \leq 0$, and
    every iteration of the loop halves $\log
        \frac{\eta_{\mathrm{hi}}}{\eta_{\mathrm{lo}}}$, therefore we make at most
    $\ceil{\log \log \eta_0 L}$ loop iterations.
    The output step-size $\eta_{\mathrm{lo}}$ satisfies
    $\frac{\eta_{\mathrm{hi}}}{2} \leq \eta_{\mathrm{lo}} \leq \eta_{hi}$ and
    $\phi(\eta_{\mathrm{lo}}) \leq 0$.
    Specializing \cref{eq:exp-search-1} for $\eta = \eta_0$ and using that
    $\phi(\eta_{\mathrm{lo}}) \leq 0$ we get
    \begin{align}
        f(\bar{x}_k) - f(\xopt) & \leq \frac{\sqn{x_0 - \xopt}_2}{2 \eta_{\mathrm{lo}} k} + \frac{\eta_{\mathrm{lo}} \sum_{i=0}^k M(x_i (\eta_{\mathrm{lo}}), x_{i+1} (\eta_{\mathrm{lo}})) \sqn{\nabla f(x_i (\eta_{\mathrm{lo}}))}_2}{2k} \cdot \phi(\eta_{\mathrm{lo}}) \nonumber \\
        \label{eq:exp-search-proof-1}
                                & \leq \frac{\sqn{x_0 - \xopt}_2}{2 \eta_{\mathrm{lo}} k}.
    \end{align}
    By the loop invariant $\phi(\eta_{\mathrm{hi}}) > 0$ we have
    \begin{align*}
        \phi(\eta_{\mathrm{hi}}) > 0 \Leftrightarrow \eta_{\mathrm{hi}} > \frac{\sum_{i=0}^K \sqn{\nabla f(x_i (\eta_{\mathrm{hi}}))}_2}{\sum_{i=0}^K \sqn{\nabla f(x_i (\eta_{\mathrm{hi}}))}_2 M(x_i (\eta_{\mathrm{hi}}), x_{i+1} (\eta_{\mathrm{hi}}))}
    \end{align*}
    By the loop termination condition we have $\eta_{\mathrm{lo}} \geq \frac{\eta_{\mathrm{hi}}}{2}$, combining this with the last equation we get
    \begin{align*}
        \eta_{\mathrm{lo}} \geq  \frac{\eta_{\mathrm{hi}}}{2} \geq \frac{1}{2} \frac{\sum_{i=0}^K \sqn{\nabla f(x_i (\eta_{\mathrm{hi}}))}_2}{\sum_{i=0}^K \sqn{\nabla f(x_i (\eta_{\mathrm{hi}}))}_2 M(x_i (\eta_{\mathrm{hi}}), x_{i+1} (\eta_{\mathrm{hi}}))}.
    \end{align*}
    Plugging this into \cref{eq:exp-search-proof-1} we obtain
    \begin{align*}
        f(\bar{x}_k) - f(\xopt) \leq \frac{\sqn{x_0 - \xopt}_2}{k} \cdot \frac{\sum_{i=0}^K \sqn{\nabla f(x_i (\eta_{\mathrm{hi}}))}_2 M(x_i (\eta_{\mathrm{hi}}), x_{i+1} (\eta_{\mathrm{hi}}))}{\sum_{i=0}^K \sqn{\nabla f(x_i (\eta_{\mathrm{hi}}))}_2}
    \end{align*}
    It remains to notice that $\eta_{\mathrm{hi}} \in [\eta_{\mathrm{lo}}, 2 \eta_{\mathrm{lo}}]$.
\end{proof}

\polyak*

For the proof of this theorem, we will need the following proposition:
\begin{proposition}\label{prop:polyak}
    Let $x \in \mathbb{R}^d$. Define $\eta_x = \gamma \frac{f(x) - f(\xopt)}{\norm{\nabla f(x)}^2}$ for some $\gamma \in (1, 2)$ and let $\tilde{x} = x - \eta_x \nabla f(x)$. Then,
    \begin{align*}
        f(x) - f(\xopt) & \geq \frac{\gamma-1}{\gamma^2} \frac{2}{M(x, \tilde{x})} \sqn{\nabla f(x)}_2.
    \end{align*}
\end{proposition}
\begin{proof}
    Observe
    \begin{align}
        f(x) - f(\xopt) & = f(x) - f(\tilde{x}) + f(\tilde{x}) - f(\xopt) \nonumber \\
        \label{eq:9}
                        & \geq f(x) - f(\tilde{x}).
    \end{align}
    By smoothness we have
    \begin{align*}
        f(\tilde{x}) & \leq f(x) + \abr{\nabla f(x), \tilde{x} - x} + \frac{M(x, \tilde{x})}{2} \norm{\tilde{x} - x}^2 \\
                     & = f(x) - \eta_x \sqn{\nabla f(x)}_2 + \frac{\eta_x^2 M(x, \tilde{x})}{2} \sqn{\nabla f(x)}_2.
    \end{align*}
    Plugging back into \cref{eq:9} we get
    \begin{align*}
        f(x) - f(\xopt) & \geq \eta_x \sqn{\nabla f(x)}_2 - \frac{\eta_x^2 M(x, \tilde{x})}{2} \sqn{\nabla f(x)}_2.
    \end{align*}
    Let us now use the definition of $\eta_x = \gamma \frac{f(x) - f(\xopt)}{\sqn{\nabla f(x)}_2}$ to get
    \begin{align*}
        f(x) - f(\xopt) & \geq \gamma (f(x) - f(\xopt)) - \frac{\gamma\eta_x M(x, \tilde{x})}{2} (f(x) - f(\xopt)).
    \end{align*}
    Assuming that $f(x) \neq f(\xopt)$ then we get by cancellation
    \begin{align*}
        1 & \geq \gamma - \frac{\gamma\eta_x M(x, \tilde{x})}{2}.
    \end{align*}
    Using the definition of $\eta_x$ again
    \begin{align*}
        1 - \gamma & \geq - \gamma^2 \frac{M(x, \tilde{x})}{2} \frac{f(x) - f(\xopt)}{\sqn{\nabla f(x)}_2}
    \end{align*}
    Rearranging we get
    \begin{align*}
        f(x) - f(\xopt) & \geq \frac{\gamma-1}{\gamma^2} \frac{2}{M(x, \tilde{x})} \sqn{\nabla f(x)}_2.
    \end{align*}
    If $f(x) = f(\xopt)$ then $\sqn{\nabla f(x)}_2 = 0$, both sides are identically zero and the statement holds trivially.
\end{proof}

Now we can prove our theorem on the convergence of GD with Polyak step-sizes:

\begin{proof}[Proof of Theorem~\ref{thm:polyak}]
    We start by considering the distance to the optimum and expanding the square
    \begin{align}
        \sqn{x_{k+1} - \xopt}_2 & = \sqn{x_k - \xopt}_2 + 2 \abr{x_{k+1}-x_k, x_k-\xopt} + \sqn{x_{k+1} - x_k}_2 \nonumber            \\
        \label{eq:11}
                                & = \sqn{x_k - \xopt}_2 - 2 \eta_k \abr{\nabla f(x_k), x_k - \xopt} + \eta_k^2 \sqn{\nabla f(x_k)}_2.
    \end{align}
    Let $\delta_k = f(x_k) - f(\xopt)$. By convexity we have $f(\xopt) \geq f(x_k) + \abr{\nabla f(x_k), \xopt - x_k}$. Therefore we can upper bound \cref{eq:11} as
    \begin{align}
        \sqn{x_{k+1} - \xopt}_2 & \leq \sqn{x_k - \xopt}_2 - 2 \eta_k \delta_k + \eta_k^2 \sqn{\nabla f(x_k)}_2 \nonumber                                                          \\
                                & = \sqn{x_k - \xopt}_2 - 2 \eta_k \delta_k + \eta_k \left( \gamma \frac{\delta_k}{\sqn{\nabla f(x_k)}_2}  \right) \sqn{\nabla f(x_k)}_2 \nonumber \\
        \label{eq:12}
                                & = \sqn{x_k - \xopt}_2- (2-\gamma) \eta_k \delta_k,
    \end{align}
    where in the second line we used the definition of $\eta_k$. By \Cref{prop:polyak} we have
    \begin{align}
        \label{eq:10}
        \delta_k \geq \frac{\gamma-1}{\gamma} \frac{2}{M(\xk, \xkk)}  \sqn{\nabla f(x_k)}_2.
    \end{align}
    Using this in \cref{eq:12} gives
    \begin{align*}
        \sqn{x_{k+1} - \xopt}_2 & \leq \sqn{x_k - \xopt}_2 - (2-\gamma) \eta_k \frac{\gamma-1}{\gamma^2} \frac{2}{M(\xk, \xkk)} \sqn{\nabla f(x_k)}_2                                                     \\
                                & = \sqn{x_k - \xopt}_2 - (2-\gamma) \frac{\gamma-1}{\gamma^2} \frac{2}{M(\xk, \xkk)} \left( \gamma \frac{\delta_k}{\sqn{\nabla f(x_k)}_2}  \right) \sqn{\nabla f(x_k)}_2 \\
                                & = \sqn{x_k - \xopt}_2 - \frac{2 (2-\gamma) (\gamma-1)}{\gamma M(\xk, \xkk)}  \delta_k.
    \end{align*}
    Rearranging we get
    \begin{align*}
        \frac{2 (2-\gamma) (\gamma-1)}{\gamma M(\xk, \xkk)} \delta_k \leq \sqn{x_k - \xopt}_2 - \sqn{x_{k+1} - \xopt}_2.
    \end{align*}
    Summing up and telescoping we get
    \begin{align*}
        \sum_{i=0}^{k-1} \frac{2 (2-\gamma) (\gamma-1)}{\gamma M(x_i, x_{i+1})} \delta_i \leq \sqn{x_0 - \xopt}_2.
    \end{align*}
    Let $\bar{x}_k = \frac{1}{\sum_{i=0}^{k-1} M(x_i, x_{i+1})^{-1}} \sum_{i=0}^{k-1} M(x_i, x_{i+1})^{-1} x_i$, then by the convexity of $f$ and Jensen's inequality we have
    \begin{align*}
        f(\bar{x}_k) - f(\xopt)
         & \leq \frac{1}{\sum_{i=0}^{k-1} M(x_i, x_{i+1})^{-1}} \sum_{i=0}^{k-1} M(x_i, x_{i+1})^{-1} \delta_i             \\
         & \leq \frac{\gamma }{2(2-\gamma)(\gamma-1)} \frac{1}{\sum_{i=0}^{k-1} M(x_i, x_{i+1})^{-1}} \sqn{x_0 - \xopt}_2.
    \end{align*}
\end{proof}

\begin{theorem}\label{thm:polyak-alternate}
    If $f$ is convex and differentiable,
    then GD with the Polyak step-size and \( \gamma < 2 \) satisfies,
    \begin{equation}
        f(\bar x_k) - f(\xopt)
         \leq \frac{1}{(2 - \gamma) \sum_{i=0}^k \eta_i}
        \norm{x_0 - \xopt}_2^2,
    \end{equation}
    where
    \( \bar x_k = \sum_{i=0}^{k-1} \eta_i  x_i / \rbr{\sum_{i=0}^{k-1} \eta_i} \).
\end{theorem}
\begin{proof}
    The proof begins in the same manner as that for \cref{thm:polyak},
    \begin{align*}
        \sqn{x_{k+1} - \xopt}_2
         & = \sqn{x_k - \xopt}_2 + 2 \abr{x_{k+1}-x_k, x_k-\xopt} + \sqn{x_{k+1} - x_k}_2 \\
         & = \sqn{x_k - \xopt}_2 - 2 \eta_k \abr{\nabla f(x_k), x_k - \xopt}
        + \eta_k^2 \sqn{\nabla f(x_k)}_2                                                  \\
         & \leq \sqn{x_k - \xopt}_2 - 2 \eta_k \delta_k + \eta_k^2 \sqn{\nabla f(x_k)}_2  \\
         & = \sqn{x_k - \xopt}_2 - 2 \eta_k \delta_k
        + \eta_k \left( \gamma \frac{\delta_k}{\sqn{\nabla f(x_k)}_2}  \right)
        \sqn{\nabla f(x_k)}_2                                                             \\
         & = \sqn{x_k - \xopt}_2- (2-\gamma) \eta_k \delta_k.
        \intertext{Re-arranging, summing from \( i = 0 \) to \( k-1 \), and
            dividing by \( \sum_{i=0}^{k-1} \eta_i \),}
        \implies
        \sum_{i=0}^{k-1} \frac{\eta_i}{\sum_{i=0}^k \eta_i}\rbr{f(x_i) - f(\wopt)}
         & \leq \frac{1}{(2 - \gamma) \sum_{i=0}^k \eta_i}
        \sbr{\norm{x_0 - \xopt}_2^2 - \norm{x_k - \xopt}_2^2}                             \\
        \implies
        f(\bar x_k) - f(\xopt)
         & \leq \frac{1}{(2 - \gamma) \sum_{i=0}^k \eta_i}
        \norm{x_0 - \xopt}_2^2,
    \end{align*}
    which completes the proof.
\end{proof}



\begin{lemma}\label{lemma:ngd-intermediate}
    Normalized GD with step-sizes $\eta_k$ satisfies
    \begin{equation}\label{eq:151}
        - \frac{\eta_k}{\norm{\nabla f(\xk)}_2} \ev{\nabla f(\xk), \nabla f(x_{k+1})}
        \leq  \etak^2 M(\xk,\xkk)- \etak \norm{\nabla f(\xk)}_2.
    \end{equation}
\end{lemma}
\begin{proof}
    By convexity we have
    \begin{align}
        f(\xkk)
         & \leq f(\xk) + \ev{ \xkk-\xk, \nabla f(\xkk)}                                                                                \nonumber \\
         & = f(\xk) - \frac{\etak}{\norm{\nabla f(\xk)}_2} \ev{\nabla f(\xk), \nabla f(\xkk)}      \label{eq:step1NGD}
    \end{align}
    Now note that
    \begin{align}
        - \frac{\etak}{\norm{\nabla f(\xk)}_2} \ev{\nabla f(\xk), \nabla f(\xkk)}
         & =   \frac{\etak}{\norm{\nabla f(\xk)}_2} \ev{\nabla f(\xk), \nabla f(\xk) - \nabla f(\xkk)}        \\
         & \hspace{4em} - \etak \norm{\nabla f(\xk)}_2 \nonumber                                              \\
         & \leq  \etak \|  \nabla f(\xk) - \nabla f(\xkk)\| - \etak \norm{\nabla f(\xk)}_2 	,	  \label{eq:15}
    \end{align}
    where we used Cauchy-Schwarz.
    Recalling the definition of directional smoothness
    \begin{align*}
        M(\xk,\xkk) \eqdef \frac{\norm{\nabla f(\xk) - \nabla f(\xkk)}}{\norm{\xk - \xkk}} = \frac{\norm{\nabla f(\xk) - \nabla f(\xkk)}}{\etak}
    \end{align*}
    in \cref{eq:15} gives
    \begin{align*}
        - \frac{\etak}{\norm{\nabla f(\xk)}_2} \ev{\nabla f(\xk), \nabla f(\xkk)}
         & \leq  \etak^2 M(\xk,\xkk)- \etak \norm{\nabla f(\xk)}_2.
    \end{align*}
\end{proof}

\ngd*
\begin{proof}
    Here we will first establish that for any non-increasing sequence of step-sizes $\etak>0$ we have that
    \begin{equation}\label{eq:norm-gd-master-eq}
        \min_{i\in[k-1]} f(x_i)-f(\xopt)
        \le
        \frac{1}{2}\frac{\Delta_0+ \sum_{i=0}^{k-1}\eta_i^2}{k}
        \rbr{\frac{f(x_0)}{k\eta_0^2} - \frac{f(\xopt)}{k\eta_{k-1}^2}
            + \sum_{i=0}^{k-1}\frac{M(\x_i, \x_{i+1})}{k}}.
    \end{equation}
    The specialized results follow by assuming that
    $\sum_{i=0}^{k-1} \frac{M(\x_i, \x_{i+1})}{k}$ is bounded, which it is
    the case of $L$--Lipschitz gradients.
    In particular the
    $\min_{i\in[k-1]} f(\x_i)-f(\xopt) \in \mathcal{O}(1/T)$ result follows by
    plugging in $\eta_i =1/\sqrt{k}$ and using that
    \begin{align*}
        \sum_{i=0}^{k-1}\eta_i^2
         & =  \sum_{i=0}^{k-1}\frac{1}{k} = 1 \\
        \frac{f(x_0)}{k\eta_0^2} - \frac{f(\xopt)}{k\eta_{k-1}^2}
         & = f(x_0) - f(\xopt).
    \end{align*}
    Alternatively we get
    $\min_{i\in[k-1]} f(\x_i)-f(\xopt) \in \mathcal{O}(\log(T)/T)$
    by plugging in $\eta_i = 1/\sqrt{i+1}$ and using that
    \begin{align*}
        \sum_{i=0}^{k-1}\eta_i^2
         & =  \sum_{i=0}^{k-1}\frac{1}{i+1} \leq \log(k) \\
        \frac{f(x_0)}{k\eta_0^2} - \frac{f(\xopt)}{k\eta_{k-1}^2}
         & =  \frac{f(x_0)}{k}- f(\xopt).
    \end{align*}
    With this in mind, let us prove \cref{eq:norm-gd-master-eq}.

    By convexity,
    \begin{align*}
        f(\xkk)
         & \le f(\xk) - \frac{\etak}{\norm{\nabla f(\xk)}_2} \nabla f(\xk)^\top \nabla f(\xkk)                             \\
         & \le f(\xk) -\etak \norm{\nabla f(\xk)}_2 + \etak^2 M(\xk,\xkk).                     \tag{Using~\eqref{eq:151} }
    \end{align*}
    Re-arranging, dividing through by $\etak^2$, and then summing over $i=0,\cdots,k-1$ gives
    \begin{align}
        \sum_{i=0}^{k-1}  \frac{\norm{\nabla f(\x_i)}_2}{\eta_i}
         & \le \frac{f(x_0)}{\eta_0^2}
        + \sum_{i=1}^{k-2} f(\x_i)\left(\frac{1}{\eta_i^2}
        - \frac{1}{\eta_{i-1}^2}\right)- \frac{f(\xopt)}{\eta_{k-1}^2}
        + \sum_{i=0}^{k-1} M(\x_i, \x_{i+1}) \nonumber                  \\
         & \leq \frac{f(x_0)}{\eta_0^2} - \frac{f(\xopt)}{\eta_{k-1}^2}
        + \sum_{i=0}^{k-1}   M(\x_i, \x_{i+1}), \label{eq:slienz8o4z23}
    \end{align}
    where we used that
    $\eta_{i-1} \leq \eta_i \implies \frac{1}{\eta_i^2} -\frac{1}{\eta_{i-1}^2} \leq 0$.
    Using Jensen's inequality over the map $a \mapsto 1/a$, which is convex
    for $a$ positive, gives
    \begin{align}
        \sum_{i=0}^{k-1} \frac{\eta_i}{\norm{\nabla f(\xk)}_2}
        \ge \frac{k^2}{\sum_{i=0}^{k-1} \norm{\nabla f(\xk)}_2/\eta_i}
         & \overset{\eqref{eq:slienz8o4z23}}{\ge}
        \frac{k^2}{\frac{f(x_0)}{\eta_0^2} - \frac{f(\xopt)}{\eta_{k-1}^2}
            + \sum_{i=0}^{k-1} M(\x_i, \x_{i+1})}. \label{eq:tnreslinisere}
    \end{align}
    Meanwhile, recall our notation $\Delta_{i} = \norm{\x_i - \xopt}_2^2$.
    Expanding the squares and using that $f(x)$ is convex, we have that
    \begin{align*}
        \Delta_{i+1}
         & = \Delta_{i} - 2\frac{\eta_i}{\norm{\nabla f(\x_i)}}_2 \nabla f(\x_i)^\top (\x_i-\xopt) + \eta_i^2 \\
         & \le \Delta_i -2\eta_i \frac{f(\x_i)-f(\xopt)}{\norm{\nabla f(\xk)}_2} + \eta_i^2.
    \end{align*}
    As before, we use $\delta_i := f(\x_i)-f(\xopt)$.
    Re-arranging, summing both sides of the above over
    $i=0, \ldots, k-1$ and using telescopic cancellation gives
    \begin{align*}
        \sum_{i=0}^{k-1}\eta_i \frac{\delta_i}{\|\nabla f(\x_i)\|}
        \le \frac{\Delta_0 + \sum_{i=0}^{t-1}\eta_i^2}{2}.
    \end{align*}
    Using the above along with~\eqref{eq:tnreslinisere} gives,
    \begin{align*}
        \min_{i \in [k-1]} \delta_i
         & \le \frac{1}{\sum_{i=0}^{k-1} \frac{\eta_i}{\norm{\nabla f(\x_i)}_2}}
        \sum_{i=0}^{k-1} \eta_i \frac{\delta_i}{\norm{\nabla f(\x_i)}_2}         \\
         & \leq \frac{1}{2}\frac{\Delta_0 + \sum_{i=0}^{k-1}\eta_i^2}
        {\sum_{i=0}^{k-1} \frac{\eta_i}{\|\nabla f(\x_i)\| }}                    \\
         & \leq \frac{1}{2}\frac{\Delta_0 + \sum_{i=0}^{k-1}\eta_i^2}{k}
        \rbr{\frac{f(x_0)}{k\eta_0^2} - \frac{f(\xopt)}{k\eta_{k-1}^2} + \sum_{i=0}^{k-1} \frac{M(\x_i, \x_{i+1})}{k}}
    \end{align*}

\end{proof}

%% file: appendices/experiments.tex

In this section we provide additional details necessary to reproduce our experiments.
We run our logistic regression experiments using PyTorch \citep{paszke2019pytorch}.
For the UCI datasets, we use the pre-processed version of the data
provided by \citet{fernandez2014we}, although we do not use their evaluation
procedure as it is known have test-set leakage.
Instead, we randomly perform an 80--20 train-test split and use the test
set for validation.
Unless otherwise stated, all methods are initialized using the
Kaiming initialization \citep{he2015delving}, which is standard in PyTorch.

In order to compute the strongly adapted step-sizes, we run the SciPy
\citep{2020SciPy-NMeth} implementation of Newton method on
\cref{eq:zero-finding}.
In general, we find this procedure is surprisingly robust, although it can
be slow.

\textbf{\cref{fig:rate-comparison}}:
We pick two datasets from the UCI repository to showcase different
behaviors of the upper-bounds.
We compute a tight-upper bound on \( L \) as follows.
Recall that for logistic regression problems the Hessian is given by
\[
	\nabla^2 f(x) = A^\top \text{Diag}\rbr{\frac{1}{\sigma(-y \cdot Ax) + 2 + \sigma(y \cdot Ax)}} A,
\]
where \( A \) is the data matrix and \( \sigma(z) = \frac{1}{1 + \exp(z)} \)
is the sigmoid function.
A short calculation shows that the diagonal matrix
\[
	\text{Diag}\rbr{\frac{1}{\sigma(-y \cdot Ax) + 2 + \sigma(y \cdot Ax)}} \preceq \frac{1}{4} \bfI,
\]
which is tight when \( x = 0 \).
As a result, \( L = \lambda_{\text{max}}(A^\top A) / 4 \).
We compute this manually.
We also compute the optimal value for the logistic regression problem using the
SciPy implementation of BFGS \citep{liu1989limited}.
We use this value for \( f(\xopt) \) to compute the Polyak step-size and
when plotting sub-optimality.
It turns out that the upper-bound based on \( L \)-smoothness for both
GD with the Polyak step-size \citep{hazan2019revisiting} and standard GD
\citep{bubeck2015convex} is
\[
	f(\xk) - f(\xopt) \leq \frac{2L \norm{x_0 - \xopt}_2^2}{k}.
\]

\textbf{\cref{fig:quadratics}}:
We run these experiments using vanilla NumPy.
As mentioned in the text, we generate a quadratic optimization problem
\[
	\min_{x} \half x^\top A x  - b^\top x,
\]
where the eigenvalues of \( A \) were generated to follow power law
distribution with parameter \( \alpha = 3 \).
We scaled the eigenvalues to ensure \( L = 1000 \).
The dimension of the problem we create is \( d = 300 \).
We repeat the experiment for \( 20 \) random trials and plot the mean and
standard deviations.

\textbf{\cref{fig:logistic-comparison}}:
We pick three different datasets from the UCI repository to showcase the
possible convergence behavior of the optimization methods.
We compute \( L \) and \( f(\wopt) \) as described above for
\cref{fig:rate-comparison}.
For normalized GD, we use the step-size schedule \( \etak = \eta_0 / \sqrt{k} \)
as suggested by our theory.
To pick \( \eta_0 \), we run a grid search on the grid generated by
\texttt{np.logspace(-8, 1, 20)}.
We implement AdGD from scratch and use a starting step-size of 
\( \eta_0 = 10^{-3} \).
We use the same procedure to compute the strongly adapted step-sizes as
described above.

%% file: appendices/computation.tex

The experiments in \cref{fig:quadratics} were run on a MacBook Pro
(16 inch, 2019) with a 2.6 GHz 6-Core Intel i7 CPU and 16GB of memory.
All other experiments were run on a Slurm cluster with several different node
configurations.
Our experiments on the cluster were run with nodes using
(i) Nvidia A100 GPUs (80GB or 40GB memory) or Nvidia H100-80GB GPUs
with Icelake CPUs, or (ii)
Nvidia V100-32GB or V100-16GB GPUs with Skylake CPUs.
All jobs were allocated a single GPU and 24GB of RAM.